\documentclass[journal,twoside,web]{ieeecolor}
\usepackage{arxiv_generic}
\usepackage{cite}
\usepackage{amsmath,amssymb,amsfonts}
\usepackage[ruled,vlined]{algorithm2e}
\usepackage{graphicx}
\usepackage{textcomp}
\usepackage{hyperref}
\usepackage{times}
\usepackage{bbm}

\newtheorem{assumption}{Assumption}
\newtheorem{claim}{Claim}
\newtheorem{remark}{Remark}
\newtheorem{theorem}{Theorem}
\newtheorem{definition}{Definition}
\newtheorem{proposition}{Proposition}
\newtheorem{lemma}{Lemma}

\newcommand{\bR}{\mathbb{R}}
\newcommand{\bI}{\mathbbm{1}}
\newcommand{\bE}{\mathbb{E}}
\newcommand{\bP}{\mathbb{P}}
\newcommand{\cS}{\mathcal{S}}
\newcommand{\cX}{\mathcal{X}}
\newcommand{\cT}{\mathcal{T}}
\newcommand{\cE}{\mathcal{E}}
\newcommand{\wV}{\widehat{V}}
\newcommand{\rkhsnorm}{\bar{\nu}}
\newcommand{\poly}{\mbox{\textit{poly}}}
\newcommand\init{\mathrel{\stackrel{\makebox[0pt]{\mbox{\normalfont\tiny iid}}}{\sim}}}

\def\BibTeX{{\rm B\kern-.05em{\sc i\kern-.025em b}\kern-.08em
    T\kern-.1667em\lower.7ex\hbox{E}\kern-.125emX}}
\begin{document}
\title{Sample Complexity and Overparameterization Bounds for Temporal Difference Learning with Neural Network Approximation}
\author{Semih Cayci, Siddhartha Satpathi, Niao He, and R. Srikant, \IEEEmembership{Fellow, IEEE}
% \thanks{This paragraph of the first footnote will contain the date on 
% which you submitted your paper for review. It will also contain support 
% information, including sponsor and financial support acknowledgment. For 
% example, ``This work was supported in part by the U.S. Department of 
% Commerce under Grant BS123456.'' }
\thanks{S. Cayci and S. Satpathi are with Coordinated Science Laboratory at the University of Illinois at Urbana-Champaign, Urbana, IL 61801, USA (e-mails: \{scayci, ssatpth2\}@illinois.edu).}
\thanks{N. He is with the Department of Computer Science at ETH Zurich, Zurich 8006, Switzerland (e-mail: niao.he@inf.ethz.ch).}
\thanks{R. Srikant is with c3.ai Digital Transformation Institute, the Department of Electrical and Computer Engineering and Coordinated Science Laboratory at the University of Illinois at Urbana-Champaign, Urbana, IL 61801, USA  (e-mail: rsrikant@illinois.edu).}}

\maketitle

\begin{abstract}
 
In this paper, we study the dynamics of temporal difference learning with neural network-based value function approximation over a general state space, namely, \emph{Neural TD learning}. We consider two practically used algorithms, projection-free and max-norm regularized Neural TD learning, and establish the first convergence bounds for these algorithms. An interesting observation from our results is that max-norm regularization can dramatically improve the performance of TD learning algorithms, both in terms of sample complexity and overparameterization. In particular, we prove that max-norm regularization improves state-of-the-art convergence bounds. The results in this work rely on a novel Lyapunov drift analysis of the network parameters as a stopped and controlled random process.

\end{abstract}


\section{Introduction}
Recently, deep reinforcement learning (RL) algorithms have achieved significant breakthroughs in challenging high-dimensional problems in a broad spectrum of applications including video gaming \cite{mnih2013playing, silver2017mastering, silver2017masteringc}, natural language processing \cite{li2016deep}, and robotics \cite{gu2017deep, kalashnikov2018qt}. An important component of these success stories lies in the power and versatility provided by neural networks in function approximation. Despite the impressive empirical success, the convergence properties of RL algorithms with neural network approximation are not yet fully understood due to their inherent nonlinearity. 

In this paper, we investigate the convergence of temporal-difference (TD) learning algorithm equipped with neural network approximation, namely Neural TD learning, which is an important building block of many deep RL algorithms. Convergence of TD learning with linear function approximation and least-squares approximation has been established in the literature \cite{bertsekas2011temporal, tsitsiklis1997analysis, yu2009convergence}. On the other hand, it is well-known that using nonlinear approximation may lead to divergence in TD learning \cite{tsitsiklis1997analysis}. Nonetheless, TD learning with neural network approximation is widely used in practice for policy evaluation because of its simplicity and empirical effectiveness \cite{lillicrap2015continuous, wang2019neural}. Therefore, it is important to understand and analyze the convergence properties of Neural TD learning. Recent study of overparameterized networks in the so-called neural tangent kernel (NTK) regime provided important insights in explaining the empirical success of neural networks in supervised learning \cite{jacot2018neural, arora2019fine, chizat2019lazy, allen2018learning, du2018gradient, ji2019polylogarithmic}, and thereafter reinforcement learning \cite{cai2019neural, sirignano2019asymptotics, brandfonbrener2019geometric,xu2020finite}. Despite the theoretical insights provided by recent studies, there is still a large gap between theory and practice. As we will discuss later, these prior works  either consider Neural TD learning in the infinite width limit for finite state spaces, or Neural TD learning with $\ell_2$-projection; neither of which is used in practice. Moreover, explicit characterization of the sample complexity and the amount of  overparameterization  required for Neural TD learning algorithms to approximate the true value function within arbitrary accuracy has remained elusive, which we address in this paper. 

% Recently, the theoretical analysis of overparameterized neural networks, particularly the neural tangent kernel (NTK) provided important insights in explaining the empirical success of neural networks in supervised learning applications \cite{jacot2018neural, chizat2019lazy, du2018gradient, allen2018learning, ji2019polylogarithmic}. Thereafter, the neural tangent kernel analysis was utilized to analyze the approximation capabilities of neural reinforcement learning algorithms in the overparameterized regime. The existing lit

% which utilizes an extra projection step in each iteration to bound the change of network parameters \cite{cai2019neural, xu2020finite}.  

% \textit{The main contribution of our paper is as follows: we study Neural TD learning for policy evaluation problems with possibly uncountably infinite state spaces and provide bounds on both the number of samples and the amount of overparameterization required to achieve a desired level of performance.}

\subsection{Main Contributions}
The paper presents a non-asymptotic analysis of TD learning with neural network approximation. We elaborate on some of the contributions in this paper below:

\begin{itemize}
    
    \item \textbf{Analysis of Neural TD learning:} We analyze two practically used Neural TD learning algorithms: (i) vanilla projection-free Neural TD and (ii) max-norm regularized Neural TD. We prove, for the first time, that both algorithms achieve any given target error within a provably rich function class, which is dense in the space of continuous functions over a compact state space. In particular, we establish explicit bounds on the required number of samples, step-size and network width to achieve a given target error.
    
    \item \textbf{Improved convergence bounds:} We show that projection-free and max-norm regularized Neural TD improve the prior state-of-the-art overparameterization bounds by factors of $1/\epsilon^2$ and $1/\epsilon^6$, respectively, for a given target error $\epsilon$. Notably, we prove that max-norm regularized Neural TD  achieves the sharpest overparameterization and sample complexity bounds in the literature, which theoretically supports its empirical effectiveness.
    
    \item \textbf{Key insights on regularization:} Our analysis reveals that using regularization based on $\ell_\infty$ geometry leads to considerably improved overparameterization and sample complexity bounds compared to the $\ell_2$-regularization over a provably rich function class in the NTK regime.
    
    % For any given $\epsilon > 0$, we provide explicit bounds on the required width $m$, learning rate $\alpha$, and number of samples $T$ to achieve target error $\epsilon$ by neural TD-learning. In particular, our analysis identifies the norm of the value function in the RKHS induced by the NTK as the main quantity that determines the required $m$ and $T$.
    
    \item \textbf{Analytical techniques:} We propose a novel Lyapunov drift analysis to track the evolution of neural network parameters and the error simultaneously using martingale concentration and stopping times. Our technique can be of independent interest to the analysis of other stochastic approximation algorithms and deep RL methods. 
    
    % The analysis introduced in this paper to track the trajectory of network weights using martingale concentration and stopping times, applied to the drift of a Lyapunov function, can be useful in the analysis of other commonly used neural RL methods.

    % \item Extendable analysis: Prior work on the analysis of deep reinforcement learning methods heavily relies on an extra projection step, which is not employed in most practical deep RL applications. The analysis methods introduced in this paper to track the trajectory of network weights, based on the theory of martingale concentration and stopping times, can be extended to analyze other deep RL methods commonly used in practice, such as deep Q-learning and neural policy gradient.
    
\end{itemize}

\subsection{Comparison with Previous Results}
Variants of Neural TD learning have been analyzed in the literature. For a quantitative comparison in terms of the required sample complexity and overparameterization bounds to achieve a given target error, please see Table \ref{table:comparison}.

The first result on the convergence of Neural TD learning was presented in \cite{cai2019neural}. Their work builds upon the analysis in \cite{tsitsiklis1997analysis,bhandari2018finite,du2018gradient,arora2019fine} and requires constraining the network parameter within a compact set through the $\ell_2$-projection at each iteration. They prove convergence to a stationary point
% \NH{stationary point or gloablly optimal point?} \SC{They use both but "global optimality" sounds like some kind of overselling.} 
in a \textit{random} function class $\mathcal{F}_{B,m}$ where $m$ is the network width and $B$ is a given projection radius. Consequently, the algorithm suffers from an approximation error $\epsilon_m = O(\bE[\|V-\Pi_{\mathcal{F}_{B,m}}V\|_\mu])$, which is not explicitly bounded, and possibly non-vanishing even with increasing width and projection radius. 
% \NH{Shouldn't it vanish as $m,B\to\infty$? Also, this random function class $\mathcal{F}_{B,m}$ is never defined anywhere.}\SC{Interestingly no. $\epsilon_\infty$ is incurred by $f_0(\cdot)\sim\mathcal{GP}(0,k)$ for some kernel $k$ as $m\rightarrow\infty$. The error can be bounded by $O(\log(1/\delta))$ with high probability over random initialization, but it does not vanish. $\mathcal{F}_{B,\infty} = f_0(\cdot) \oplus \mathcal{F}_{NTK}$ in \cite{wang2019neural, cai2019neural}. We resolve it by symmetric initialization, which sets $f_0 = 0$ for any $m,R$.} 
It is shown in \cite{cai2019neural, wang2019neural} that this variant of Neural TD learning with projection, equipped with a ReLU network of width $O(1/\epsilon^8)$ achieves an error $\epsilon+\epsilon_m$ after $O(1/\epsilon^4)$ iterations. Unlike \cite{cai2019neural, wang2019neural}, our Neural TD learning algorithms converge to the \emph{true} value function in a provably rich function class without any approximation error. We show that the algorithms that we consider in this paper achieve improved overparameterization bounds $\widetilde{O}(1/\epsilon^6)$ and $\widetilde{O}(1/\epsilon^2)$ for a given target error $\epsilon$, which improve the existing results by $1/\epsilon^2$ to $1/\epsilon^6$.

In practice, projection-free \cite{mnih2013playing} and max-norm regularized \cite{srivastava2014dropout, goodfellow2013maxout, srebro2005maximum} algorithms are often adopted in training neural networks because of their computational efficiency and expressive power, which we consider in this work. In contrast, the Neural TD with $\ell_2$-projection considered in \cite{cai2019neural, wang2019neural} can be computationally expensive for high-dimensional state-spaces as it cannot be performed in parallel.

Projection-free Neural TD learning has also been considered in \cite{agazzi2019temporal, brandfonbrener2019geometric}; however, these works only deal with finite state-space problems in the infinite-width regime, i.e., they do not provide bounds on the amount of overparameterization required. Since these results rely on the positive definiteness of the limiting kernel, the required overparameterization is much larger than the size of the state space which negates the benefits of Neural TD learning over tabular TD learning.

Our work is related to the analysis of (stochastic) gradient descent in the NTK regime. It is shown in \cite{du2018gradient, jacot2018neural} that the network parameters trained by gradient descent lie inside a ball around their initialization. However, they require massive overparameterization to ensure the positive definiteness of the neural tangent kernel, which would imply finite state and width much larger than the size of the state space in Neural TD learning. To establish such a result for stochastic gradient descent (and with modest overparameterization) requires additional work, and this problem has been considered for supervised learning tasks in \cite{ji2019polylogarithmic, oymak2019overparameterized}. Our paper uses an analysis technique inspired by \cite{ji2019polylogarithmic}, but deviates from this line of work as we consider TD learning over an infinite state space, which has significantly different dynamics than supervised learning.

% In particular, they characterize the random time at which the network parameters escape a ball of desired width around the random initialization. However, the fact that we are considering TD learning, which has different dynamics than supervised learning problems, requires a new set of tools to characterize the said random time. Specifically, we use tools from stochastic Lyapunov and martingale theory to obtain a bound on the random time with high probability and use it to further bound the error in the value function approximation.

\begin{table*}[t]
\begin{center}\begin{tabular}{ |p{2.6cm}||p{1.3cm}|p{1.3cm}|p{1.4cm}|p{1cm}|p{3cm}|  }
 \hline
 Paper& State space & Network width &Sample complexity&Error&Regularization\\
 \hline
 Cai et al. \cite{cai2019neural}   &General& $O(1/\epsilon^8)$    &$O(1/\epsilon^4)$&   $\epsilon +\epsilon_{m}$&$\ell_2$-projection \\
 Wang et al.\cite{wang2019neural} &General&   $O(1/\epsilon^8)$  & $O(1/\epsilon^4)$   &$\epsilon+\epsilon_{\infty}$&$\ell_2$-projection\\
 Agazzi \& Lu \cite{agazzi2019temporal} &Finite&$poly(|\cX|)$ & $O(\log(1/\epsilon))$&  $\epsilon$&$\poly(|\cX|)$ width\\
 This paper (PF-NTD) &General &$\widetilde{O}(1/\epsilon^6)$ & $O(1/\epsilon^6)$&  $\epsilon$&Early stopping\\
  This paper (MN-NTD)&General &$\widetilde{O}(1/\epsilon^2)$ & $O(1/\epsilon^4)$&  $\epsilon$&Max-norm projection\\
 \hline
\end{tabular}
\end{center}
 \caption{The overparameterization and sample complexity bounds for neural TD-learning algorithms. PF-NTD denotes projection-free, MN-NTD denotes max-norm regularized Neural TD learning algorithm. $\epsilon_{m}=\bE\|V-\Pi_{\mathcal{F}_{B,m}}V\|_\pi$ denotes the approximation error of the random function class $\mathcal{F}_{B,m}$ for a given value function $V$.}
%  \NH{I suggest to expand the table across two columns, and also include  differences in the algorithm (projection-free, $\ell_2$ or max-norm regularization), differences in settings (finite state or general state space, finite or infinite width), or possibly other details ( high prob bound vs expectation).}   
\label{table:comparison}
\end{table*}

\subsection{Notation}
For any index set $\mathcal{I}$ and set of vectors $\{b_i\in\bR^d:i\in\mathcal{I}\}$, we denote $[b_i]_{i\in\mathcal{I}}\in\bR^{d|\mathcal{I}|}$ as the vector that is created by the concatenation of $\{b_i:i\in\mathcal{I}\}$. For an event $\cE$ and random variable $X$, $\bE[X;\cE] = \bE[X\cdot \bI_\cE]$. For any integer $n\geq 1$, $[n] = \{1,2,\ldots, n\}$. For any vector $x\in\bR^d$ and $\rho > 0$, $\mathcal{B}(x, \rho)$ denotes the ball in $\bR^d$ with radius $\rho$ centered at $x$.

\section{System Model}
For simplicity,  we consider a Markov reward process $\{(s_t, r_t): t = 0,1,\ldots\}$, where the Markov chain $s_t$ takes on values in the state space $\mathcal{S}$, and there is an associated reward $r_t = r(s_t)$ in every time-step for a reward function $r:\mathcal{S}\rightarrow [0,1]$. The process $\{s_t:t\geq 0\}$ evolves according to the transition probabilities $P(s, A) = \mathbb{P}(s_{t+1}\in A|s_t=s)$ for any $s\in\mathcal{S}$, $A\subset \cS$ and $t\geq 0$. We assume that the Markov chain $\{s_t:t\geq 0\}$ is an ergodic unichain, therefore there exists a stationary probability distribution $\pi$: $$\pi(A) = \lim_{t\rightarrow\infty}\mathbb{P}(s_t \in A|s_0 = s),~\forall s\in\mathcal{S},A\subset \cS.$$ The value function associated with the Markov reward process $\{(s_t, r_t):t\geq 0\}$ is defined as follows:
\begin{equation}
    V(s) = \mathbb{E}\Big[\sum_{t=1}^\infty\gamma^t r_t|s_0=s\Big],~\forall s \in \mathcal{S},
\end{equation}
where $\gamma \in (0, 1)$ is the discount factor. The Bellman operator for this Markov reward process, denoted by $\mathcal{T}$, is defined as follows:
\begin{equation}
    \mathcal{T}\widehat{V}(s) = r(s)+\gamma\int_{s^\prime\in\mathcal{S}}\widehat{V}(s^\prime)P(s, ds^\prime),~\forall s\in\mathcal{S}.
\end{equation}
The value function $V$ is the fixed point of the Bellman operator $\mathcal{T}$: $V(s) = \mathcal{T}V(s)$ for all $s\in\mathcal{S}$. If the state space $\cS$ is large, or countably or uncountably infinite, the direct solution of the so-called Bellman equation is computationally inefficient, thus approximation methods are used to evaluate $V$. In this paper, we study the problem of approximating value functions using neural networks given samples from the Markov reward process. We note that the Markov reward process is typically obtained by applying a stationary policy to a controlled Markov process.

For simplicity, we consider independent and identically distributed samples from the stationary distribution $\pi$ of the Markov chain in this paper. Namely, at time $t$, we obtain an observation vector $(s_t, s_t^\prime)$ where $s_t\sim \pi$ and $s_t^\prime \sim P(s_t,\cdot)$. We denote $\mathcal{F}_t = \sigma(\{(s_j,s_j^\prime):j=0,1,\ldots,t\})$ to be the history up to (including) time $t$. The case where the samples are generated by the Markov reward process can be handled as in \cite{srikant2019finite}, but we do not consider that here. 

In a broad class of reinforcement learning applications, each state $s\in\cS$ is represented by a $d$-dimensional vector $\psi(s)$ where $\psi:\mathcal{S}\rightarrow \mathbb{R}^d$. For example, in \cite{mnih2013playing}, each state in an Atari game is represented by the corresponding high-dimensional raw image data, while the positions of the players on the board are encoded as a high-dimensional state vector in \cite{silver2017mastering, silver2017masteringc}. For a given trajectory $\{s_t:t\geq 0\}$, we denote the state representations by $x_t = \psi(s_t)$ for $t\geq 0$. We denote the space of state representations as $\mathcal{X} = \{x\in\mathbb{R}^d: x = \psi(s),s\in\cS\}$, and use $V(x), r(x), \pi(x),$ etc. to denote the quantities related to a state $\psi^{-1}(x)\in\cS$ with a slight abuse of notation. Without loss of generality, we make the following assumption on the state representation, which is commonly used in the neural network literature \cite{ji2019polylogarithmic, arora2019fine, oymak2019overparameterized, cai2019neural}.

\begin{assumption}
For any state $s\in\cS$, we assume $\|\psi(s)\|_2 \leq 1$ and $\|x\|_2\leq 1$ for all $x\in\cX$.
\label{assumption:norm}
\end{assumption}

In the next subsection, we introduce the neural network architecture that will be used to approximate the value function.
\subsection{Neural Network Architecture for Value Function Approximation}
Throughout the paper, we consider the two-layer ReLU network to approximate the value function $V$:
% \begin{align}
% \begin{aligned}
%     Q(x; W, a) &= \frac{1}{\sqrt{m}}\sum_{i=1}^ma_i\sigma(W_i^\top  x + b_i ), %\\&= \frac{1}{\sqrt{m}}\sum_{i=1}^ma_i\mathbb{I}\{W_i^\top  x\geq 0\}W_i^\top  x,
%     \label{eqn:q-network}
%     \end{aligned}
% \end{align}
\begin{align}
\begin{aligned}
    Q(x; W, a) &= \frac{1}{\sqrt{m}}\sum_{i=1}^ma_i\sigma(W_i^\top  x ) \\&= \frac{1}{\sqrt{m}}\sum_{i=1}^ma_i\mathbb{I}\{W_i^\top  x\geq 0\}W_i^\top  x.
    % \label{eqn:q-network}
    \end{aligned}
\end{align}
\noindent where $\sigma(z) = \max\{0, z\} = z\cdot \mathbb{I}\{z \geq 0\}$ is the ReLU activation function, $a_i\in\mathbb{R}$ and $W_i\in\mathbb{R}^d$ for $i\in[m]$. We include a bias term in $W_i$'s, and express $x$ as $(x,c)$ for a constant $c \in (0,1)$.

\textbf{Symmetric initialization:} The NTK regime is established by random initialization, and various initialization schemes are used, as a common example $a_i \init \textrm{Unif}\{-1,+1\}$ and $W_i(0)\init \mathcal{N}(0, I_d)$ \cite{ji2019polylogarithmic, oymak2020toward}. In this paper, we consider an almost-equivalent symmetric variant of this initialization for the sake of simplicity, which was proposed in \cite{bai2019beyond}: $a_i=-a_{i+m/2}\init \textrm{Unif}\{-1,+1\}$ and $W_i(0)=W_{i+m/2}(0)\init \mathcal{N}(0, I_d)$ independent and identically distributed over $i=1,2,\ldots,m/2$, and independent from each other. The additional benefit of the symmetric initialization is that it provides $Q(x;W(0),a) = 0$ with probability 1 for all $x\in\cX$. Without symmetric initialization, $\lim_{m\rightarrow\infty}Q(x; W(0), a)$ acts like a random noise term, which leads to an additional approximation error \cite{mjt_dlt}. We fix $a_i$ as initialized, and update $W_i(t)$ by using gradient steps, as in \cite{li2018learning, du2018gradient, arora2019fine}. The sigma field generated by $\{a_i,W_i(0):i\in[m]\}$ is denoted as $\mathcal{F}_{init}$.

\textbf{Function class: }Define the space $$\mathcal{H} = \Big\{v:\mathbb{R}^d\rightarrow\mathbb{R}^d ~\big|~ \bE\big[\|v(w_0)\|_2^2\big]<\infty,w_0\sim\mathcal{N}(0,I_d)\Big\}.$$ 
We assume that the value function $V$ lies in the following function class.
\begin{assumption}
There exists a vector $v\in\mathcal{H}$ and $\rkhsnorm \geq 0$ such that:
\begin{equation}
    V(x) = \bE[v^\top (w_0)\phi(x;w_0)],~w_0\sim\mathcal{N}(0,I_d),~\forall x \in\mathcal{X},
\end{equation}
where $\sup_{w\in\bR^d}\|v(w)\|_2\leq \rkhsnorm$ and $\phi(x;w) = \mathbb{I}\{w^\top x \geq 0\}x$.
\label{assumption:realizability}
\end{assumption}
%\begin{lemma}
%Let $K(x,y) = \psi^{\top}(x)\psi(y)=\sum_{i=0}^{\infty}d_i(x^{\top}y)^i,$ where $\psi(x) := [\sqrt{d_0},\sqrt{d_1}x, \sqrt{d_2}x^{\otimes 2}, \sqrt{d_3}x^{\otimes 3},\ldots]$ and $V(x) = \bE[v^\top (w_0)\phi(x;w_0)]= \psi^{\top}(x)w^*,$ then $\|V(x)\|_{H} =\bE\big[\|v(w_0)\|_2^2\big] =  \|w^*\|^2$ and any continuous function $f$ can be approximated arbitrarily closely by $V(x)$ with $\|w^*\|^2<\infty.$  
%\end{lemma}
\begin{remark}\normalfont If we replace the condition $\sup_{w\in\bR^d}\|v(w)\|_2\leq \rkhsnorm$ in Assumption~\ref{assumption:realizability} by $\bE\big[\|v(w_0)\|_2^2\big]<\infty,$ then it implies that $V$ belongs to the reproducing kernel Hilbert space (RKHS) induced by the
Neural Tangent Kernel (NTK) corresponding to the infinite width neural network given by 
\begin{align}
\begin{aligned}
K(x, y) &= \bE[\phi(x;w_0)^\top \phi(y;w_0)] \\&= \bE[\bI\{w_0^\top x\geq 0\}\bI\{w_0^\top y\geq 0\}x^\top y],
\label{eqn:ntk}
\end{aligned}
\end{align}
with the inner product between functions $f(.) = \bE[u^\top(w_0)\phi(.; w_0)]$ and $g(.) = \bE[v^\top(w_0)\phi(.; w_0)]$ is defined as $\langle f, g\rangle_{\tt NTK} = \bE[u^\top(w_0)v(w_0)]$ \cite{rahimi2008uniform}. The above kernel can be shown to be a universal kernel \cite{ji2019neural} and hence the RKHS induced by the NTK is dense in the space of continuous functions on compact set $\mathcal{X}$ \cite{micchelli2006universal}. Therefore, it is possible to replace Assumption~\ref{assumption:realizability} by the more general assumption that $V$ is continuous on a compact state space $\cX$. In this case, from \cite[Theorem 4.3]{ji2019neural}, we know that one can find a function $\tilde{V}$ in the RKHS associated with the NTK, i.e.,
$\tilde{V}(x) = \bE[\tilde{v}^\top (w_0)\phi(x;w_0)],~\forall x \in\mathcal{X},$
such that $\sup_w \|\tilde{v}(w)\|_2\leq \overline{\nu}$ for some finite $\overline{\nu}$ which approximates $V,$ where $\overline{\nu}$
depends on the approximation error $\sup_x |V(x)-\tilde{V}(x)|.$ If we replace Assumption~\ref{assumption:realizability} by the assumption that $V$ is continuous, then the results later can be modified to reflect this approximation error.

\end{remark} 

\begin{remark}\normalfont
 We note $\sup_{w\in\bR^d}\|v(w)\|_2\leq \rkhsnorm$ in Assumption~\ref{assumption:realizability} implies that $\bE\big[\|v(w_0)\|_2^2\big]\leq \rkhsnorm,$ thus $\overline{\nu}$ is an upper bound on the RKHS norm of $V$ when it lies in the RKHS.
\end{remark}

\begin{remark}\normalfont It is worth noting the difference between our work and the projection-free TD learning work in \cite{agazzi2019temporal, brandfonbrener2019geometric}. They consider a finite state space in the infinite width limit. For finite $\cX$, choosing $m = \poly(|\cX|)$ guarantees that the kernel $K$ is strictly positive-definite \cite{du2018gradient, arora2019fine}, thus in the infinite width limit, the minimum eigenvalue of the limiting kernel is bounded away from zero. By extending the NTK analysis in \cite{du2018gradient}, one can guarantee with further overparameterization that the empirical kernel $$\widehat{K}_t(x, y) = \frac{1}{m}\sum_{i=1}^m\bI\{W_i(t)^\top x\geq 0\}\bI\{W_i(t)^\top y\geq 0\}x^\top y,$$ under TD learning dynamics is also positive definite, thus it can be shown that the network parameters satisfy $W_i(t)\in\mathcal{B}(W_i(0),\rho/\sqrt{m})$ for some $\rho < \infty$ for all $t\geq 1.$ However, such a massive overparameterization, i.e., $m = \Omega(|\cX|^p)$ for some $p \geq 1$, is not meaningful for TD learning with function approximation, because one may use tabular TD learning directly instead. Thus, we do not seek to make the kernel $K$ positive definite by massive overparameterization in this work since we consider a general (potentially infinite) state space $\cX$. Instead, by Assumption~\ref{assumption:realizability}, we consider functions that can be realized in the RKHS induced by the NTK, and quantify the required overparameterization in terms of $\rkhsnorm$, a bound on the RKHS norm of $V$.
\label{rem:ntk}
\end{remark}

In the next subsection, we present TD learning algorithms to approximate the value function $V$ by a neural network $Q(.; W, a)$.

\section{Neural Temporal Difference Learning Algorithms}\label{sec:algorithms}

For a given function $\mu=[\mu(x)]_{x\in\cX}$, we denote the weighted $\ell_2$-norm of any function $\widehat{V}$ as: $$\|\widehat{V}\|_\mu = \sqrt{\int_{x\in\cX}|\widehat{V}(x)|^2\mu(dx)}.$$ TD learning aims to minimize mean-squared Bellman error, which is defined as follows:
\begin{align}
\begin{aligned}
    L(W, a) &= \|Q(W, a) - \mathcal{T}Q(W, a)\|_\pi^2 \\&= \int_{x\in\cX}\Big(Q(x;W,a)-\mathcal{T} Q(x;W,a)\Big)^2\pi(dx),
    \end{aligned}
    \label{eqn:bellman-error}
\end{align}
for any $W_i\in\mathbb{R}^d,a_i\in\mathbb{R}$ for $i=1,2,\ldots,m$, where ${Q}(W,a) = [Q(x;W,a)]_{x\in\cX},$ $\pi$ is the stationary distribution of the Markov chain, and $\mathcal{T}$ is the Bellman operator.

Given the initialization $\{\big(a_i,W_i(0)\big):i\in[m]\}$, the parameter update is performed as follows:
\begin{equation*}
    W(t+1/2) = W(t) + \alpha\Big(r_t + \gamma Q_t(x_t^\prime)-Q_t(x_t)\Big) \nabla_WQ_t(x_t),
\end{equation*}
where $\alpha > 0$ is the step-size, $Q_t(x) = Q(x; W(t), a)$ is the network at time step $t \geq 0$. The algorithm is summarized in Algorithm \ref{alg:neural-td}. We consider two variants of the Neural TD learning algorithm:

(1) \textbf{Projection-free Neural TD learning (PF-NTD):} The network parameters are updated as follows:
    \begin{equation}
        W(t+1) = W(t+1/2).
    \end{equation}
    For regularization, we utilize early stopping, i.e., the number of samples $T$ is chosen as a function of the problem parameters and target error, which we will specify in Theorem \ref{thm:neural-td-exp}.   Note that the expressive power of the neural network approximation is fully exploited in PF-NTD.

(2) \textbf{Max-norm regularized Neural TD learning (MN-NTD):} For a given parameter $R>0$, let the set of parameters for max-norm regularization be defined as:
  \begin{equation}
        \mathcal{G}_{m,R}^i = \{W_i\in\bR^{d}:\|W_i-W_i(0)\|_2\leq \frac{R}{\sqrt{m}}\}, \forall i\in[m].
    \end{equation}
    % Let $\Pi_{\mathcal{G}_{m,R}}W = \arg\min_{\hat{w}\in \mathcal{G}_{m,R}}\|W-\hat{w}\|_2$ for any $W\in\bR^{md}$. 
    Then, the network parameters are updated as follows:
    \begin{equation}
        W_i(t+1) = \Pi_{\mathcal{G}_{m,R}^i}W_i(t+1/2), \forall i\in[m].
        \label{eqn:max-norm}
    \end{equation}
    % \NH{I changed the notation here to emphasize that this can be done in parallel. Otherwise, it might look too similar to the $\ell_2$ projection. Could you make sure the they are consistent in other places.} 
    where $\Pi_{\mathcal{G}}(\cdot)$ is the projection operator onto set $\mathcal{G}$.
%     \begin{equation}
%       W_i(t+1) = \hskip -0.075cm \begin{cases} 
%       W_i(t+\frac{1}{2}), \hskip 0.4cm \|W_i(t+\frac{1}{2})-W_i(0)\|_2 \leq \frac{R}{\sqrt{m}}, \\
%       \frac{ W_i(t+\frac{1}{2})\frac{R}{\sqrt{m}}}{\| W_i(t+\frac{1}{2})\|_2}, \hskip 0.3cm else,
%   \end{cases}
%   \label{eqn:max-norm}
%     \end{equation}
% for any $i \in [m]$. We denote the above transformation as $W(t+1) = \Pi_{\mathcal{G}_{m,R}}W(t+1/2)$. 

% \NH{The projection step in (9) needs to be re-centered.}
    
    % let
    % \begin{equation}
    %     \mathcal{G}_{m,R} = \{W\in\bR^{md}:\|W_i-W_i(0)\|_2\leq R/\sqrt{m}\},
    % \end{equation}
    % and let $\Pi_{\mathcal{G}_{m,R}}W = \arg\min_{\hat{w}\in \mathcal{G}_{m,R}}\|W-\hat{w}\|_2$ for any $W\in\bR^{md}$. Then, the network parameters are updated as follows:
    % \begin{equation}
    %     W(t+1) = \Pi_{\mathcal{G}_{m,R}}W(t+1/2),t=0,1,\ldots,T-1.
    %     \label{eqn:max-norm}
    % \end{equation}
    Max-norm regularization was introduced in \cite{srebro2005rank, srebro2005maximum}, and has been widely used in training neural networks \cite{srivastava2014dropout, goodfellow2013maxout}. Note that unlike the $\ell_2$-projection in \cite{wang2019neural, cai2019neural}, max-norm regularization in \eqref{eqn:max-norm} can be performed in parallel for all neurons $i\in[m]$, which makes it computationally more feasible. Furthermore, it implies projection onto a well-chosen subset, which leads to much sharper overparameterization and sample complexity bounds for a given value function $V$ \cite{cai2019neural, wang2019neural} as we will show in Theorem \ref{thm:mn-td}. Therefore, it is practically used in training neural networks \cite{srivastava2014dropout, goodfellow2013maxout}. On the other hand, the choice of $R$ drastically impacts the expressive power of the neural network approximation, and a too small choice of $R$ may incur an approximation error unlike PF-NTD. For TD learning, we will specify the choice of $R$ as a function of the smoothness of the value function $V$ for convergence in Theorem \ref{thm:mn-td}.

\begin{algorithm}[t]
\SetAlgoLined
 Initialization: { $-a_i = a_{i+m/2}\sim \textrm{Unif}\{-1,+1\}, \newline W_i(0)=W_{i+m/2}(0)\sim\mathcal{N}(0,I_d)$, $\forall i\in[\frac{m}{2}]$}\\
 \For{$t < T-1$}{
  Observe $x_t = \psi(s_t), r_t = r(s_t)$ and $x_t^\prime = \psi(s_t^\prime)$ where $(s_t,s_t^\prime)\init\pi\circ P(s_t,\cdot )$\\
  Compute stochastic semi-gradient: $g_t = \big(r_t+\gamma Q_t(x_t^\prime)-Q_t(x_t)\big)\nabla_WQ_t(x_t)$\\
  Take a semi-gradient step: $W(t+1/2) = W(t)+\alpha g_t$\\
  \If{projection-free}{
    $W(t+1) = W(t+1/2)$;}
  \If{max-norm regularization}{
  $W_i(t+1)=\Pi_{\mathcal{G}_{m,R}^i}W_i(t+1/2), \forall i\in[m]$; }
  Update iterate: $\widehat{W}(t+1) = \Big(1-\frac{1}{t+2}\Big)\widehat{W}(t)+\frac{1}{t+2}W(t+1)$\\
 }
 Output: $\overline{Q}_T(x) = Q(x;\widehat{W}(T-1),a)$ for all $x\in\cX$\\
    % Output: $\overline{Q}_T(x) = Q(x;W(T-1),a)$ for all $x\in\cX$\\
 \caption{PF/MN-Neural TD Learning}
 \label{alg:neural-td}
\end{algorithm}

% We also consider a mean-path version of the Neural TD Learning Algorithm. In this case, the semi-gradient follows the mean path:
% \begin{equation}
%     \overline{g}_t = \bE_t[(\cT Q_t(x_t) - Q_t(x_t))\nabla_WQ_t(x_t)],
%     \label{eqn:semi-gradient}
% \end{equation}
% where $\bE_t[X] = \bE[X|\mathcal{F}_{t-1}]$ for any $t$. As such, the parameter update is performed according to the mean semi-gradient: $W(t+1) = W(t) + \alpha \overline{g}_t$.

\section{Main Results}
In the following, we consider a general (possibly infinite) state-space $\cX$, and present our main result on the performance of Neural TD learning algorithms described in Section \ref{sec:algorithms}.

\subsection{Performance of Projection-Free Neural TD Learning}
In the following, we present the sample complexity and overparameterization bounds of PF-NTD. The proof of this result is presented in Section \ref{sec:neural-td}.
\begin{theorem}
Under Assumptions \ref{assumption:norm} and \ref{assumption:realizability}, for any (possibly infinite) state-space $\cX$, target error $\epsilon > 0$ and error probability $\delta \in (0,1)$, let $\ell(m,\delta) = 4\sqrt{d\log(2m+1)}+4\sqrt{\log(1/\delta)}$, $\lambda = \frac{3\overline{\nu}^2}{(1-\gamma)\epsilon\delta}$, $$m_0 = \frac{16\Big(\rkhsnorm+\big(\lambda+\ell(m_0,\delta)\big)\big(\rkhsnorm+\lambda\big)\Big)^2}{(1-\gamma)^2\epsilon^2},$$ and
\begin{equation*}
    \quad\alpha_0 = \frac{(1-\gamma)\epsilon^2}{(1+2\lambda)^2} \min\Big\{\frac{\lambda^2}{32\rkhsnorm^2(\sqrt{d}+\sqrt{2\log(m_0/\delta)})^2},1\Big\}.
\end{equation*} Then, for any width $m\geq m_0$,  PF-NTD with step-size $\alpha \leq \alpha_0$ yields the following bound after $T = \frac{\rkhsnorm^2}{4\alpha(1-\gamma)\epsilon^2}$ iterations:
\begin{align}
    \bE\Big[\big\|\overline{Q}_T - V\big\|_\pi;\mathcal{E}_T\Big]
    \leq \frac{1}{T}\sum_{t<T}\bE[\|Q_t-V\|_\pi;\cE_T]+\epsilon
    % \frac{1}{T}\sum_{t<T}\bE\Big[\|Q_t - V\|_\pi;\mathcal{E}_T\Big] \leq \frac{1}{T}\sum_{t<T}\sqrt{\bE\Big[\|Q_t - V\|_\pi^2;\mathcal{E}_T\Big]}
    \leq 4\epsilon,
    \label{eqn:error-bound}
\end{align}
where $Q_t = [Q_t(x)]_{x\in\cX}$, $V = [V(x)]_{x\in\cX}$, and the expectation is over the random trajectory and random initialization, and the event $\mathcal{E}_T$ is defined as:
\begin{equation*}
    \cE_T=\Big\{\max_{i\in[m]}\|W_i(t)-W_i(0)\|_2 \leq \frac{\lambda}{\sqrt{m}}, t < T\Big\}\cap E_1,
\end{equation*}
for some $E_1\in\mathcal{F}_{init}$, which satisfies $\bP(\cE_T) > 1-4\delta$.
\label{thm:neural-td-exp}
\end{theorem}

Theorem \ref{thm:neural-td-exp} implies that there exists a set $\cE_T$ of trajectories which occurs with probability at least $1-4\delta$ such that Algorithm \ref{alg:neural-td} achieves target error $\epsilon$ under the event $\cE_T$ for sufficiently large number of samples and overparameterization.
% \NH{Something doesn't look quite right here. $T=O(\frac{1}{\alpha\epsilon^2})$, $\alpha=O(\frac{\epsilon^2}{\lambda^2})$, $\lambda=O(\frac{1}{\epsilon\delta})$. This implies that $T=O(\frac{1}{\epsilon^6\delta^2})$, $m=O(\frac{1}{\epsilon^6\delta^4})$? Also, I am bit confused, should we use big $O$ or big $\Omega$ notation here? We need at most $m_0$ samples and at most $T_0$ iterates to achieve $\epsilon$ accuracy.} 
Note that Theorem \ref{thm:neural-td-exp} can be interpreted as $m = \widetilde{O}(T/\delta^2)$ where $T=poly(\rkhsnorm/\delta)O(1/\epsilon^6)$ is the number of samples since the neural network processes one sample per iteration. With this interpretation, we observe that regularization is obtained by overparameterization with respect to $T$, the number of samples, akin to the classical NTK results in the literature \cite{jacot2018neural, du2018gradient, arora2019fine}. The overparameterization bound has polynomial dependence on the number of samples and does not scale with the size of state-space. Unlike~\cite{cai2019neural,wang2019neural}, our error bound does not contain any additional approximation error terms.

 We have the following remark on the main challenges in the proof of Theorem \ref{thm:neural-td-exp}.

\begin{remark}\normalfont
 In \cite{cai2019neural}, projection is applied to the network parameters $W(t)$ in each TD learning iteration to keep $W(t)$ inside a ball of a given radius around the random initialization $W(0)$. In the proof of Theorem \ref{thm:neural-td-exp}, we propose methods based on a novel use of Lyapunov drift coupled with martingale concentration to track the evolution of $\|W_i(t)-W_i(0)\|_2$ and the approximation error $\|Q_t-V\|_\pi$ simultaneously.
\end{remark}
 
 \subsection{Performance of Max-Norm Regularized Neural TD Learning}
 In the following, we present the overparameterization and sample complexity bounds for MN-NTD.
 \begin{theorem}
 Under Assumptions \ref{assumption:norm}-\ref{assumption:realizability}, for any error probability $\delta \in (0, 1)$, let $\ell(m,\delta) = 4\sqrt{d\log(2m+1)}+4\sqrt{\log(1/\delta)},$ and $R > \rkhsnorm$. Then, for any target error $\epsilon > 0$, number of iterations $T \in \mathbb{N}$, network width $$m > \frac{16\Big(\rkhsnorm + \big(R+\ell(m,\delta)\big)\big(\rkhsnorm+R\big)\Big)^2}{(1-\gamma)^2\epsilon^2},$$ and step-size $$\alpha = \frac{\epsilon^2(1-\gamma)}{(1+2R)^2},$$ MN-NTD yields the following bound:
 \begin{equation*}
     \bE\Big[\|\overline{Q}_T-V\|_\pi;E_1\Big] \leq \frac{(1+2R)\rkhsnorm}{\epsilon\sqrt{T}} + 3\epsilon,
 \end{equation*}
 where $E_1\in\mathcal{F}_{init}$ holds with probability at least $1-\delta$.
 \label{thm:mn-td}
 \end{theorem}
 
 The proof of Theorem \ref{thm:mn-td} is similar to, and simpler than the proof of Theorem \ref{thm:neural-td-exp} because the max-norm constraint strictly controls the movement of the parameters. The proof can be found in Appendix \ref{app:mn-ntd}.
%  \NH{Need to ask some discussions here.}

% \begin{remark}[Comparison of PF-NTD and MN-NTD]
\subsection{Remarks}
The above Theorems \ref{thm:neural-td-exp} and \ref{thm:mn-td} provide, to the best of our knowledge, the first explicit characterization of the sample complexity and overparametrization required for PF-NTD and MN-NTD to converge to the true value function with target error $\epsilon$. Below we list some further implications.
% \begin{enumerate}

\emph{$\ell_2$ vs. $\ell_\infty$ regularizations: }Both PF-NTD and MN-NTD yield improved bounds on $m$ compared to the algorithms in \cite{cai2019neural,wang2019neural} over the provably rich NTK function class (see Table \ref{table:comparison}). A key insight from our analysis is that this improvement is mainly because both PF-NTD and MN-NTD are designed to control $\max_{i\in[m]}\|W_i(t)-W_i(0)\|_2$ via the choice of the stopping time  (PF-NTD) or max-norm projection (MN-NTD), while the regularization method in \cite{cai2019neural, wang2019neural} is designed to control $\|W(t)-W(0)\|_2$. Notably, NTD with max-norm regularization achieves the sharpest overparameterization and sample complexity bounds among all NTD variants, which justifies the empirical success of max-norm regularization for training ReLU networks in practice \cite{srivastava2014dropout, goodfellow2013maxout}.
    
\emph{Approximation power:} PF-NTD fully exploits the expressive power of the neural network approximation in practice since the parameters are not strictly constrained. On the other hand, MN-NTD confines the network parameters within the sets $\mathcal{G}_{m,R}^i$ with a fixed radius $R/\sqrt{m}$, which may limit the expressive power of the neural network, especially for small radius $R$. A similar loss of approximation power arise for the projection-based NTD studied in \cite{cai2019neural, wang2019neural} for the same reason. However, we note that if the value function lies within the NTK class and $\nu<R,$ i.e., the neural network is expressive enough to represent the value function, then MN-NTD has an advantage over PF-NTD.
    
\emph{Convergence rate: }Regularization of PF-NTD relies on early stopping, whereas MN-NTD utilizes more aggressive max-norm regularization. Without any strict control over $\max_{i\in[m]}\|W_i(t)-W_i(0)\|_2$, PF-NTD requires considerably smaller step-sizes for convergence. Consequently, the sample complexity and required width for PF-NTD to achieve a target error $\epsilon$ are worse than MN-NTD for which larger step-sizes can be chosen.

\section{Analysis of the Neural TD Learning Algorithm}\label{sec:neural-td}
In this section, we will prove Theorem \ref{thm:neural-td-exp}. Before starting the proof, let us define a quantity that will be central throughout the proof. 

\begin{definition}
For $\lambda$ as given in Theorem \ref{thm:neural-td-exp}, let
\begin{equation}
    t_1 = \inf\Big\{t> 0:\max_{i\in[m]}\|W_i(t)-W_i(0)\|_2 > \frac{\lambda}{\sqrt{m}}\Big\},
\end{equation}
be the stopping time at which there exists $i\in[m]$ such that $W_i(t)\notin\mathcal{B}\big(W_i(0),\lambda/\sqrt{m}\big)$ for the first time.
\end{definition}
Since the updates, $g_t$, are random in the Neural TD Learning Algorithm (see Algorithm \ref{alg:neural-td}), the stopping time $t_1$ is random, which constitutes the main challenge in the proof. As we will show, for any $t < t_1$, the drift of $W(t)$ can be controlled. Therefore, we will prove that $t_1 > T$ with high probability to prove the error bounds in Theorem \ref{thm:neural-td-exp}.

\textbf{Proof outline:} Below, we outline the proof steps for Theorem \ref{thm:neural-td-exp}.
\begin{enumerate}
    \item First, we will prove a drift bound for $\|W(t)-\overline{W}\|_2$ which holds for all $t < t_1$ where $\overline{W}\in\bR^{md}$ is a weight vector such that $\nabla_W^\top Q_0(x)\overline{W} \approx V(x)$ for all $x\in\cX$. 
    \item In the second step, we will use the drift bound obtained in the first step in conjunction with a stopped martingale concentration argument to show that $t_1 \geq T$ occurs with high probability, thus the drift bound holds for all $t < T$ under that event.
    \item Finally, we will use the drift bound again to show that the approximation error is bounded as in Theorem \ref{thm:neural-td-exp} under the high-probability event considered in Step 2.
\end{enumerate}

\subsection{Step 1: Lyapunov Drift bound for $W(t)$} 
We first prove a drift bound on the weight vector $W(t)$, a common step in the analysis of stochastic gradient descent and TD learning with function approximation \cite{cai2019neural, ji2019polylogarithmic, bhandari2018finite,xu2020finite}. Define the point of attraction as follows:
\begin{equation} 
\overline{W} = \Big[W_i(0)+a_i\frac{ v\big(W_i(0)\big)}{\sqrt{m}}\Big]_{i\in[m]},
\label{eqn:opt-parameter}
\end{equation}
where $W(0)$ is the initial weight vector. Intuitively, $\lim_{m\rightarrow\infty}\nabla_W^\top Q_0(x)\overline{W}= V(x)$ for any $x\in\cX$ under the symmetric initialization, which guarantees $\nabla_W^\top Q_0(x)W(0) = Q_0(x) = 0$ for all $x\in\cX$. For error probability $\delta \in (0,1)$, recall that we define $\ell(\delta,m) = 4\sqrt{d\log(m+1)}+4\sqrt{\log(1/\delta)}$, and let
\begin{align*}
    E_1 &= \Big\{\sup_{x\in\cX}\frac{1}{m}\sum_{i=1}^m\bI\{|W_i^\top(0)x|\leq \frac{\lambda}{\sqrt{m}}\}\leq \frac{\lambda+\ell(m,\delta)}{\sqrt{m}}\Big\},
    % E_2 &= \big\{|V-\nabla_W^\top Q_0 \overline{U}\|_\pi \leq \frac{2\rkhsnorm(1+\sqrt{2\log(1/\delta)})}{\sqrt{m}}\big\}
\end{align*}
and $\cE_t = E_1\cap \{t<t_1\}$ for any $t<T$.

The following key proposition is used to establish the drift bound. 
\begin{proposition}
Denote $\Delta_t = r_t+\gamma Q_t(x_t^\prime) - Q_t(x_t)$ as the Bellman error. Under Assumptions \ref{assumption:norm}-\ref{assumption:realizability}, we have the following inequalities:

(1) $\bE\big[\Delta_t \big(Q_t(x_t)-V(x_t)\big);\cE_t\big] \leq -(1-\gamma)z_t^2$.

(2) $\bE\big[\Delta_t \big(V(x_t)-\nabla_W^\top Q_0(x_t)\overline{W}\big);\cE_t\big] \leq \frac{4\rkhsnorm}{\sqrt{m}}z_t$,

(3) For $\ell(m, \delta)$ defined in Theorem \ref{thm:neural-td-exp}: \begin{multline}\bE\big[\Delta_t \big(\nabla_WQ_0(x_t)-\nabla_WQ_t(x_t)\big)^\top \overline{W}; \cE_t\big] \\ \leq \frac{4\big(\rkhsnorm+\lambda\big)\big(\lambda+\ell(m,\delta)\big)z_t}{\sqrt{m}},\end{multline}
where $z_t = \sqrt{\bE[\|Q_t-V\|_\pi^2;\cE_t]}$, $\bE$ is the expectation over random initialization and trajectory, $\bE_t[.] = \bE[.|\mathcal{F}_{t-1}]$ with $\mathcal{F}_{-1} = \mathcal{F}_{init}$.
\label{prop:drift-td}
\end{proposition}
The proof of Proposition \ref{prop:drift-td} is given in Appendix \ref{app:neural-td}. The first inequality in Proposition \ref{prop:drift-td} follows from the fact that the Bellman operator $\cT$ is a contraction with respect to $\|.\|_\pi$, and $V$ is the fixed point of $\cT$ \cite{tsitsiklis1997analysis}. The second inequality holds since $\nabla_W^\top Q_0(x) \overline{W}$ turns into an empirical estimate of $V$ with $m/2$ iid samples, where the variance of each term is at most $\rkhsnorm^2$. The last inequality is the most challenging one as it reflects the evolution of the network output over TD learning steps, and it is essential to have $W_i(t)\in\mathcal{B}(W_i(0),\lambda/\sqrt{m})$ to prove that part.

Now we present the main drift bound for the TD update. 

\begin{lemma}[Drift Bound]
For any $t\geq 0$, we have the following inequalities:
\begin{align}\label{eq:drift-bound}
\begin{aligned}
    \bE[\bE_t\|W(t+1)&-\overline{W}\|_2^2;t<t_1] \leq \bE[\|W(t)-\overline{W}\|_2^2;t<t_1]\\&- 2\alpha(1-\gamma)z_t^2+ \alpha^2(1+2\lambda)^2
    \\&+8\alpha z_t\Big(\frac{\rkhsnorm+\big(\rkhsnorm+\lambda\big)\big(\lambda+\ell(m,\delta)\big)}{\sqrt{m}}\Big) ,
    \end{aligned}
\end{align}
where $\overline{W}$ is as defined in \eqref{eqn:opt-parameter}, $z_t = \sqrt{\bE[\|Q_t-V\|_\pi^2;\cE_t]}$.
\label{lemma:drift-bound}
\end{lemma}
Lemma \ref{lemma:drift-bound} implies that for $t < t_1$, i.e., as long as $W_i(t)\in\mathcal{B}(W_i(0), \lambda/\sqrt{m})$ for all $i\in[m]$, the drift can be made negative by sufficiently large width $m$ and sufficiently small step-size $\alpha$.

\begin{proof}
Recall that \begin{align}
\begin{aligned}
g_t &= \big(r_t+\gamma Q_t(x_t^\prime) - Q_t(x_t)\big)\nabla_W Q_t(x_t) \\&= \Delta_t\nabla_WQ_t(x_t),\end{aligned}\label{eqn:bellman-err}
\end{align} is the semi-gradient, where $\Delta_t = r_t+\gamma Q_t(x_t^\prime) - Q_t(x_t)$ is the Bellman error. Since $W(t+1) = W(t)+\alpha{g}_t$, we have the following relation:
\begin{multline*}
    \|W(t+1)-\overline{W}\|_2^2 = \|W(t)-\overline{W}\|_2^2 + 2\alpha{\Big[{g}_t^\top\big(W(t)-\overline{W}\big)\Big]}\\+\alpha^2{\|{g}_t\|_2^2}.
\end{multline*}
 We can write the expected drift in the following form:
\begin{multline}
    \bE_t[\|W(t+1)-\overline{W}\|_2^2;\cE_t] = \|W(t)-\overline{W}\|_2^2\bI_{\cE_t} \\+ 2\alpha\underbrace{\bE_t[g_t^\top](W(t)-\overline{W})}_{(i)}\bI_{\cE_t}+\alpha^2\underbrace{\bE_t\|g_t\|_2^2}_{(ii)}\bI_{\cE_t}.
    \label{eqn:drift-td}
\end{multline}

\textbf{Bounding (i) in \eqref{eqn:drift-td}:} In order to bound (i), we expand it as follows. For any $t < t_1$:
\begin{multline}
    \bE_t[g_t^\top \big(W(t)-\overline{W}\big)] = \bE_t[\Delta_t\cdot \big(Q_t(x_t)-V(x_t)\big)] \\+ \bE_t[\Delta_t\cdot \big(V(x_t)-\nabla_W^\top Q_0(x_t)\overline{W}\big)]\\+\bE_t[\Delta_t\cdot \big(\nabla_WQ_0(x_t)-\nabla_WQ_t(x_t)\big)^\top \overline{W}],
    \label{eqn:drift-td-b}
\end{multline}
Then, we obtain the inequality in Lemma \ref{lemma:drift-bound} by applying Proposition \ref{prop:drift-td}.

\textbf{Bounding (ii) in \eqref{eqn:drift-td}:} The next argument follows the proof of \cite[Lemma 4.5]{cai2019neural}: 
    \begin{align} 
    \begin{aligned}
        \|g_t\|_2 &= \|(r_t+\gamma Q_t(x_t^\prime)- Q_t(x_t))\nabla_W Q_t(x_t)\|_2,\\
        &\leq |r_t+\gamma Q_t(x_t^\prime) - Q_t(x_t)|,\\
        &\leq 1 + 2\max_{x\in \cX}|Q_t(x)|,\\
        &\leq 1+2\|W(t)-W(0)\|_2 \leq 1+2\lambda,
        \label{eqn:gradient-norm}
        \end{aligned}
    \end{align}
    where the first inequality follows since $\|\nabla_WQ_t(x)\|_2\leq 1$ for any $t, x$, the second inequality follows since $r(x)\in[0,1]$ for all $x\in\cX$, and the last inequality holds since $|Q_t(x)| = |Q_t(x)-Q_0(x)| \leq \|W(t)-W(0)\|_2$ and $t < t_1$. Consequently, $\|g_t\|_2^2 \leq (1+2\lambda)^2$.

The result in (\ref{eq:drift-bound}) immediately follows by combining these two bounds.
\end{proof}

\subsection{Step 2: Stopping time $t_1\geq T$ with high probability} 
Now, we will use the drift result in Step 1 to show that $t_1 \geq T$ with high probability.

\begin{lemma}
Under Assumptions \ref{assumption:norm}-\ref{assumption:realizability}, we have:
\begin{equation*}
    t_1 = \inf\Big\{t > 0:\max_{i\in[m]}\|W_i(t)-W_i(0)\|_2 > \frac{\lambda}{\sqrt{m}}\Big\} \geq T,
\end{equation*}
with probability at least $1-\delta$.
\end{lemma}

\begin{proof}
First, invoking Lemma \ref{lemma:drift-bound} with the values for $T$, $\lambda$ and $m$ specified in Theorem \ref{thm:neural-td-exp}, we have the following inequality for any $t$:
\begin{multline}
    % \begin{split}
    \bE[\bE_t\|W(t+1)-\overline{W}\|_2^2;\cE_t] \leq \bE[\|W(t)-\overline{W}\|_2^2;\cE_t]\\-2\alpha(1-\gamma)z_t^2+2\alpha(1-\gamma)\epsilon^2+4\alpha(1-\gamma)\epsilon z_t,
    \label{eqn:exp-drift}
    \end{multline}
% \end{equation}
where $z_t = \sqrt{\bE[\|Q_t-V\|_\pi^2;\cE_t]}$. The step-size $\alpha$ is chosen sufficiently small so that, by \eqref{eqn:gradient-norm}, $\alpha^2\|g_t\|^2 \leq 2\alpha(1-\gamma)\epsilon^2$. 
\begin{claim}
Telescoping sum of \eqref{eqn:exp-drift} over $t < T$ yields:
\begin{equation*}
    0 \leq \rkhsnorm^2-2\alpha(1-\gamma)\sum_{t<T}(z_t-\epsilon)^2+4\alpha(1-\gamma)\epsilon^2 T.
\end{equation*}
\label{cl:telescoping-sum}
\end{claim}
\begin{proof}[Proof of Claim \ref{cl:telescoping-sum}]
Recall the notation $\bE_t[.]=\bE[.|\mathcal{F}_{t-1}]$. Let $\overline{\delta}(t) = W(t)-\overline{W}$ and $\zeta_T$ be defined as:
\begin{align*}
    \zeta_T &= \sum_{t<T}\Big(\bE_t[\|\overline{\delta}(t+1)\|_2^2]\bI_{\cE_t} - \|\overline{\delta}(t+1)\|_2^2\bI_{\cE_{t+1}}\Big), \\ 
    &\geq \sum_{t<T}\bI_{\cE_t}\Big(\bE_t[\|\overline{\delta}(t+1)\|_2^2] - \|\overline{\delta}(t+1)\|_2^2\Big) = \zeta^\prime_T, 
\end{align*} for $T \geq 1$ with $\zeta_0 = \zeta_0^\prime = 0$, where the inequality holds since $\bI_{\cE_{t+1}} \leq \bI_{\cE_t}.$ Note that $\zeta_T^\prime$ is a martingale over the filtration $\{\mathcal{F}_t\}$ since each summand constitutes a martingale difference sequence, and $\bI_{\cE_t}\in\mathcal{F}_{t-1}$ is predictable and nonnegative. Then, we have:
\begin{multline*} 
\sum_{t<T}\Big(\bE_t[\|\overline{\delta}(t+1)\|_2^2] - \|\overline{\delta}(t)\|_2^2\Big)\bI_{\cE_t}\geq \zeta_T -\rkhsnorm^2\\+ \underbrace{\|W(T)-\overline{W}\|_2^2\bI_{\cE_T^c}}_{\geq 0},
\end{multline*} 
which follows from $\|W(0)-\overline{W}\|_2 \leq \rkhsnorm$. Since $\zeta_T \geq \zeta_T^\prime$ for any $T \geq 1$, and $\zeta_T^\prime$ is a martingale with $\zeta_0^\prime = 0$, we have $\bE[\zeta_T] \geq \bE[\zeta_T^\prime] = 0$. Hence,
\begin{equation*} 
\sum_{t<T}\Big(\bE[\|W(t+1)-\overline{W}\|_2^2;\cE_t] - \bE[\|W(t)-\overline{W}\|_2^2;\cE_t]\Big) \geq -\rkhsnorm^2,
\end{equation*} 
and therefore the claim follows. 
\end{proof}

Applying Claim \ref{cl:telescoping-sum} and Jensen's inequality, we have: 
\begin{equation}
    \sum_{t<T}z_t \leq 3\epsilon T = \frac{3\rkhsnorm^2}{4\alpha(1-\gamma)\epsilon}.
    \label{eqn:total-error-td}
\end{equation}
This bound on the total error will be the fundamental quantity in the proof. Now, by using \eqref{eqn:total-error-td}, we will show that the event $\cE_T^\prime = \{t_1 < T\}\cap E_1$ occurs with low probability. For any $i\in[m]$, let $\overline{g}_i(t+1) = W_i(t+1)-W_i(t)$. Then, we have:
\begin{align}
    \nonumber \|\overline{g}_i(t+1)\|_2 &= \|\sum_{t<t_1}\overline{g}_i(t+1)\|_2,\\
    &\leq \|\sum_{t<t_1}\overline{g}_i(t+1)-\sum_{t<t_1}\bE_t[\overline{g}_i(t+1)]\|_2 \label{eqn:martingale-a} \\ &+ \|\sum_{t<t_1}\bE_t[\overline{g}_i(t+1)]\|_2.\label{eqn:martingale-b}
\end{align}

\textbf{Bounding \eqref{eqn:martingale-a}:} For any $t$, let 
\begin{equation}
    D_{i,t} = W_i(t+1)-W_i(t)-\bE_t[W_i(t+1)-W_i(t)],
\end{equation}
which forms a martingale difference sequence with respect to the filtration $\mathcal{F}_t$ since $\bE_t[D_{i,t}] = 0$. Let $X_{i, t^\prime} = \sum_{t<t^\prime}D_{i,t}$. Since $D_{i,t}$ is a martingale difference sequence, $X_{i,t}$ is a martingale. Thus, bounding \eqref{eqn:martingale-a} is equivalent to bounding $\|X_{i,t_1}\|_2$, under the event $\cE_T^\prime$. In order to achieve this, we use a concentration inequality for vector-valued martingales \cite[Theorem 2.1]{juditsky2008large}, which is given in the following.

\begin{proposition}[Concentration for Vector Martingales]
Consider a martingale difference sequence $\{D_t\in\mathbb{R}^d:t\geq 0\}$, and let $X_T = \sum_{t<T}D_t$. If $\|D_t\|_2\leq \sigma$ almost surely for all $t$, then for any $T$ and $\beta> 0$, we have the following inequality:
\begin{equation}
    \bP\Big(\|X_T\| \geq \big(\sqrt{2d}+\beta\sqrt{2}\big)\sigma\sqrt{T}\Big) \leq \exp(-\beta^2/2).
\end{equation}
\label{prop:martingale}
\end{proposition}

Since $\sup_{x\in\cX}|Q_t(x)| \leq \|W(t)-W(0)\|_2 \leq \lambda$ for all $t < t_1$, we have $\|D_{i,t}\|_2 \leq \frac{2\alpha(1+2\lambda)}{\sqrt{m}}.$  Define the stopped martingale $\widetilde{X}_{i,t} = X_{i,\min\{t,t_1\}}$, which is again a martingale with a corresponding martingale difference sequence $\widetilde{D}_{i,t}$ that satisfies $\|\widetilde{D}_{i,t}\|_2 \leq \|D_{i,t}\|_2$ \cite{williams1991probability}. Since $$\|X_{i,t_1}\|_2\cdot\mathbb{I}_{\cE_T^\prime} \leq \|\widetilde{X}_{i,T}\|_2,$$ the following inequality holds:
\begin{align*}
    \bP\Big(\|X_{i,t_1}\|_2 \geq \sqrt{2}(\sqrt{d}+\beta)\frac{2\alpha(1+2\lambda)\sqrt{T}}{\sqrt{m}};\cE_T^\prime) \leq e^{-\beta^2/2},
\end{align*} 
which follows from \begin{multline*}\{\|X_{i,t_1}\|_2 \geq \sqrt{2}(\sqrt{d}+\beta)\frac{2\alpha(1+2\lambda)\sqrt{T}}{\sqrt{m}}\}\cap \cE_T^\prime \\ \subset \{\|\widetilde{X}_{i,T}\|_2 \geq \sqrt{2}(\sqrt{d}+\beta)\frac{2\alpha(1+2\lambda)\sqrt{T}}{\sqrt{m}}\},\end{multline*} and $$\bP\Big(\|\widetilde{X}_{i,T}\|_2 \geq \sqrt{2}(\sqrt{d}+\beta)\frac{2\alpha(1+2\lambda)\sqrt{T}}{\sqrt{m}}\Big) \leq e^{-\beta^2/2},$$ by Proposition \ref{prop:martingale}. Therefore, by using union bound:
\begin{multline}
    \bP(\|X_{i,t_1}\|_2 > \big(\sqrt{2d}+2\sqrt{\log(\frac{m}{\delta})}\big)\sqrt{\frac{T}{m}} 2\alpha(1+2\lambda);\cE_T^\prime)\\\leq \delta,
    \label{eqn:stopping-time-a}
\end{multline}
The step-size $\alpha$ is chosen to satisfy $$\big(\sqrt{2d}+2\sqrt{\log(m/\delta)}\big)\sqrt{T}\cdot 2\alpha(1+2\lambda) \leq \lambda/2.$$

\textbf{Bounding \eqref{eqn:martingale-b}:} Note that we can bound \eqref{eqn:martingale-b} as follows:
\begin{equation}
        \|\sum_{t<t_1}\bE_t[\overline{g}_i(t+1)]\bI_{\cE_T^\prime}\|_2 \leq \sum_{t<t_1}\frac{2\alpha\bI_{\cE_T^\prime}}{\sqrt{m}}\|Q_t-V\|_\pi,
        \label{eqn:markov-a}
\end{equation}
for all $i\in[m]$ under $\cE_T^\prime$ since $\sup_{i,t,x}\|\nabla_{W_i}Q_t(x)\|_2\leq 1/\sqrt{m}$ (see Remark \ref{rem:bellman} for details). The expectation of the RHS above is bounded as follows:
\begin{align*}
    \frac{2\alpha}{\sqrt{m}}\bE[\sum_{t<t_1}\|Q_t-V\|_\pi\bI_{\cE_T^\prime}] &\leq \frac{2\alpha}{\sqrt{m}}\sum_{t<T}\bE[\|Q_t-V\|_\pi;\cE_t] \\&\leq \frac{2\alpha}{\sqrt{m}}\sum_{t < T}z_t,
\end{align*}
by the law of iterated expectations as the event $\{t<t_1\}\cap E_1\in\mathcal{F}_{t-1}$ as $\|W_i(t)-W_i(0)\|\in\mathcal{F}_{t-1}$. Note that the RHS of the previous inequality is upper bounded by \eqref{eqn:total-error-td}. Therefore, we have:
\begin{equation*}
    \frac{2\alpha}{\sqrt{m}}\bE[\sum_{t<t_1}\|Q_t-V\|_\pi;\cE_T^\prime] \leq \frac{6T\epsilon\alpha}{\sqrt{m}}.
\end{equation*}
Hence, we have the following:
\begin{multline*}
% \begin{aligned}
    \bigcup_{i\in[m]}\Big\{\|\sum_{t<t_1}\bE_t[\overline{g}_i(t+1)]\bI_{\cE_T^\prime}\|_2 > \frac{6\alpha T\epsilon}{\sqrt{m}\delta}\Big\}\cap\cE_T^\prime \\ \subset \Big\{\sum_{t<t_1}\frac{2\alpha\|Q_t-V\|_\pi\bI_{\cE_T^\prime}}{\sqrt{m}} > \frac{6\alpha T\epsilon}{\sqrt{m}\delta}\Big\},
    % \end{aligned}
\end{multline*}
which implies that \begin{equation}\bP(\bigcup_{i\in[m]}\Big\{\|\sum_{t<t_1}\bE_t[\overline{g}_i(t+1)]\|_2 > \frac{6\alpha T\epsilon}{\sqrt{m}\delta}\Big\};\cE_T^\prime) \leq \delta,
\label{eqn:stopping-time-b}
\end{equation}
by Markov's inequality. Now, using \eqref{eqn:stopping-time-a} and \eqref{eqn:stopping-time-b} in \eqref{eqn:martingale-a} and \eqref{eqn:martingale-b}, we conclude that $\bP(\cE_T^\prime) \leq 2\delta$. Since $\cE_T^c = \cE_T^\prime\cup E_1^c$ and $\bP(E_1^c)\leq \delta$ by Claim \ref{claim:glivenko-cantelli}, we conclude that $\cE_T$ holds with probability at least $1-3\delta$.

\end{proof}

\subsection{Step 3: Error bound} 
In Step 2, we have shown that the event $\{t_1 \geq T\}$ occurs with high probability. Since $\cE_T = \{t_1 \geq T\}\cap E_1 \subset \cE_t$ for any $t < T$, we have the following inequality:
\begin{align*}
    \bE[\|Q_t-V\|_\pi;\cE_T] &\leq \sqrt{\bE[\|Q_t-V\|_\pi^2;\cE_T]} \\ & \leq z_t = \sqrt{\bE[\|Q_t-V\|_\pi^2;\cE_t]},
\end{align*}
for any $t < T$. Consequently, by using \eqref{eqn:total-error-td} and Jensen's inequality, we have:
\begin{align*}
    \bE[\|\frac{1}{T}\sum_{t<T}Q_t-V\|_\pi;\cE_T] &\leq \frac{1}{T}\sum_{t < T}\bE[\|Q_t-V\|_\pi;\cE_T] \\&\leq \frac{1}{T}\sum_{t < T}z_t\leq 3\epsilon.
\end{align*}
In the final step, by following similar steps as \cite{cai2019neural}, we use Proposition \ref{prop:proximity} in Appendix \ref{app:neural-td} to show the proximity of $\overline{Q}_T$ and $\frac{1}{T}\sum_{t<T}Q_t$ to $\nabla_W^\top Q_0\overline{W}$, and conclude that $\bE[\|\overline{Q}_T-\frac{1}{T}\sum_{t<T}Q_t\|_\pi;\cE_T]\leq \epsilon$, which implies $\bE[\|\overline{Q}_T-V\|_\pi;\cE_T]\leq 4\epsilon$ by triangle inequality.

\section{Conclusion}
In this paper, we analyzed two practically used TD learning algorithms with neural network approximation, and established non-asymptotic bounds on the required number of samples and network width to achieve any given target error within a provably rich function class. By using a novel Lyapunov drift analysis of stopped and controlled random processes, we have shown for the first time that projection-free Neural TD learning can achieve arbitrarily small target error. In addition, we proved that max-norm regularized Neural TD learning achieves the state-of-the-art complexity bounds, which theoretically supports its empirical effectiveness in ReLU networks. One key insight from our analysis is that $\ell_\infty$-regularization yields improved results in the NTK regime compared to $\ell_2$-regularization. The extension of this work to other reinforcement learning algorithms, such as Q-learning and policy gradient methods, and different neural network architectures, such as multi-layer and convolutional networks, is left for future investigation. 

\section*{Acknowledgment}
This work was supported by NSF/USDA Grant AG 2018-67007-28379, ARO W911NF-19-1-0379, NSF CCF 1934986, NSF CCF 1704970, Navy N00014-19-1-2566.

\bibliography{root.bib}

\appendices

\section{Proofs of Section \ref{sec:neural-td}}\label{app:neural-td}

Throughout the section, we will use the following results extensively.

\begin{claim}[Lemma 4.1 in \cite{satpathi2020role}]
For any $\delta \in (0, 1)$ and $m\in\mathbb{N}$, let $$\ell_0(m,\delta) = \sqrt{8d\log(m+1)}+\sqrt{8\log(1/\delta)}.$$ Then, for any $\epsilon > 0$, $m > 10$, if $W_i\init \mathcal{N}(0, I_d)$ for all $i\in[m]$, we have:
\begin{equation}
    \sup_{x:\|x\|_2\leq 1}\frac{1}{m}\sum_{i=1}^m\bI\{|W_i^\top x|\leq \epsilon\} \leq \sqrt{\frac{2}{\pi}}\epsilon+ \frac{\ell_0(m,\delta)}{\sqrt{m}},
\end{equation}
with probability at least $1-\delta$ over the random initialization.
\label{claim:glivenko-cantelli}
\end{claim}

\begin{claim}
For any $W\in\bR^{m d}$, we have:
\begin{align*}
    \|\nabla_W Q(x; W, a) \|_2 &\leq 1,\\
    \|\nabla_{W_i} Q(x; W, a) \|_2 &\leq 1/\sqrt{m},~\forall i\in[m], 
\end{align*}
for any $x\in\cX$.
\label{cl:norm-grad}
\end{claim}
\begin{proof}
Note that $$\nabla_W Q(x; W, a) = \Big[\frac{1}{\sqrt{m}}a_i\bI\{W_i^\top(0)x\geq 0\}x\Big]_{i\in[m]},$$
which directly implies $\|\nabla_WQ(x;W,a)\|_2^2\leq \|x\|_2^2 \leq 1$ since $a_i\sim Unif\{-1,+1\}$ for all $i\in[m]$ by the symmetric initialization, and $\|x\|_2\leq 1$ for all $x\in\cX$. The other claim is proved similarly.
\end{proof}

\begin{claim}[Lemma 6.3.1 in \cite{bertsekas2011dynamic}]
For any $\widehat{V} = [\widehat{V}(s)]_{s\in\cS}$,
\begin{equation*}
    \|\mathcal{T}\widehat{V}\|_\pi \leq \gamma\cdot \|\widehat{V}\|_\pi,
\end{equation*}
where $\mathcal{T}$ is the Bellman operator.
\label{cl:contraction}
\end{claim}

\begin{claim}
For any $t \geq 0$, we have: $$\sqrt{\bE_t[(\cT Q_t(x_t)-Q_t(x_t))^2]} \leq (1+\gamma)\|Q_t-V\|_\pi.$$
\label{cl:bellman}
\end{claim}
\begin{proof}
For any $x\in\cX$, we have $\cT V(x) = V(x)$ since $V$ is the fixed point of the Bellman operator $\cT$. Therefore, we have: 
\begin{multline*}
    \sqrt{\bE_t[(\cT Q_t(x_t)-Q_t(x_t))^2]} \\= \sqrt{\bE_t[(\cT Q_t(x_t)-\cT V(x_t)-Q_t(x_t)+V(x_t))^2]}.
\end{multline*}Since $V(x), Q_t(x)\in\mathcal{F}_{t-1}$ for any given $x\in\cX$, the expectation $\bE_t$ is over $(x_t, x_t^\prime)$, thus we have: \begin{multline*}
    \sqrt{\bE_t[(\cT Q_t(x_t)-\cT V(x_t)-Q_t(x_t)+V(x_t))^2]} \\= \|\cT Q_t-\cT V - Q_t + V\|_\pi.
\end{multline*} By triangle inequality, the above inequality implies the following: $$\sqrt{\bE_t[(\cT Q_t(x_t)-Q_t(x_t))^2]} \leq \|\cT Q_t - \cT V\|_\pi + \|Q_t-V\|_\pi.$$ Since $\cT$ is a contraction over $\|.\|_\pi$ by Claim \ref{cl:contraction} with modulus $\gamma \in (0,1)$, $$\|\cT Q_t-\cT V\|_\pi \leq \gamma\|Q_t-V\|_\pi,$$ which implies the result.
\end{proof}

\begin{remark}\normalfont
 Note that for any sequence of predictable $\bR^k$-valued ($k \geq 1$) functions $h_t\in\mathcal{F}_{t-1}$ which does not depend on $x_t^\prime$, we have the following identity: \begin{align*}
\bE_t[\Delta_t h_t(x_t)] &= \int_{\cX}\Big(\int_{\cX}\Delta_t\pi(dx_t^\prime)\Big)h(x_t)\pi(dx_t),\\
&= \bE_t[\big(\cT Q_t(x_t)-Q_t(x_t)\big)h_t(x_t)\big].
\end{align*}
We use this identity extensively in the analysis throughout this work, mainly in conjunction with Claim \ref{cl:bellman}. Some examples for this are as follows: $h_t(x) = \nabla_W Q_t(x)$ which leads to $g_t = \bE_t[\Delta_th_t(x_t)]$ and $h_t(x) = \nabla_{W_i} Q_t(x)$ which leads to $W_i(t+1)-W_i(t) = \alpha\bE_t[\Delta_th_t(x_t)]$, where other instances such as $h_t(x) = Q_t(x)-V(x)$ show up in the following analysis as well.
\label{rem:bellman}
\end{remark}

% \begin{claim}[Lemma 2.4 in \cite{ji2019polylogarithmic}]
% Let $W_i(0)\sim\mathcal{N}(0, I_d)$ for any $i\in[m]$ and $x\in\cX$ where $|cX| = n$. Then, for any $\epsilon > 0$, we have:
% \begin{equation*}
%     \frac{1}{m}\sum_{i\in[m]}\mathbb{I}\{|W_i^\top(0)x|\leq \epsilon\} \leq \epsilon+\sqrt{\frac{\log(n/\delta)}{2N}},
% \end{equation*}
% simultaneously for all $x\in\cX$ with probability at least $1-\delta$ over the random initialization.
% \label{cl:separation-set}
% \end{claim}

\subsection{Proof of Proposition \ref{prop:drift-td}} 
\textbf{Part (1)} Let $\ell(\delta,m) = 2\sqrt{d\log(2m+1)}+2\sqrt{\log(1/\delta)}$,
\begin{align*}
    E_1 &= \Big\{\sup_{x\in\cX}\frac{1}{m}\sum_{i=1}^m\bI\{|W_i^\top(0)x|\leq \epsilon\}\leq \sqrt{\frac{2}{\pi}}\epsilon+\frac{\ell(m,\delta)}{\sqrt{m}}\Big\},
    % E_2 &= \big\{|V-\nabla_W^\top Q_0 \overline{U}\|_\pi \leq \frac{2\rkhsnorm(1+\sqrt{2\log(1/\delta)})}{\sqrt{m}}\big\}
\end{align*}
For any $t<T$, let $\cE_t = E_1\cap \{t<t_1\}$.
(1) For any $t<t_1$, we have the following inequality:
\begin{equation*}
    \bE_t[\Delta_t\big(Q_t(x_t)-V(x_t)\big)] \leq -(1-\gamma) \|Q_t-V\|_\pi^2.
\end{equation*}
% Hence, by the law of iterated expectations, we have: $$\bE[\Delta_t\cdot\big(Q_t(x_t)-V(x_t)\big)] \leq -(1-\gamma)\cdot\bE\|Q_t-V\|_\pi^2.$$

\begin{proof}
The proof follows the strategy first proposed in \cite{tsitsiklis1997analysis}, and then used for the convergence proofs in \cite{bhandari2018finite, cai2019neural, xu2020finite}. Let $\bE_t[.] = \bE[.|\mathcal{F}_{t-1}]$, i.e., the expectation is over $(x_t, x_t^\prime)$. Then, we have \begin{multline*}\bE_t[\Delta_t(Q_t(x_t)-V(x_t))] \\= \bE_t[(\cT Q_t(x_t)-Q_t(x_t))(Q_t(x_t)-V(x_t))],\end{multline*} by taking expectation over $x_t^\prime$ first, which implies the following:
\begin{multline*}
    \bE_t\Big[\big(\mathcal{T}Q_t(x_t)-Q_t(x_t)\big)\big(Q_t(x_t)-V(x_t)\big)\Big] \\= \bE_t\Big[\big(\mathcal{T}Q_t(x_t)-\mathcal{T}V(x_t)\big)\big(Q_t(x_t)-V(x_t)\big)\Big]\\
    - \bE_t\Big[\big(Q_t(x_t)-V(x_t)\big)\big(Q_t(x_t)-V(x_t)\big)\Big],
\end{multline*}
since $\mathcal{T}V(x) = V(x)$ for any $x\in\cX$. Therefore, we have:
\begin{multline*}
    \bE_t\Big[\big(\mathcal{T}Q_t(x_t)-Q_t(x_t)\big)\big(Q_t(x_t)-V(x_t)\big)\Big] \\ \leq \eta_t - \|Q_t-V\|_\pi^2.
\end{multline*}
where $\eta_t = \bE_t[\big(\mathcal{T}Q_t(x_t)-\mathcal{T}V(x_t)\big)\big(Q_t(x_t)-V(x_t)\big)]$. Since $\|.\|_\pi$ defines a norm, by Cauchy-Schwarz inequality, we have: \begin{align*}\eta_t &= \bE_t[\big(\mathcal{T}Q_t(x_t)-\mathcal{T}V(x_t)\big)\big(Q_t(x_t)-V(x_t)\big)] \\ &\leq \|\mathcal{T}Q_t-\mathcal{T}V\|_\pi\cdot \|Q_t-V\|_\pi.\end{align*} From Claim \ref{cl:contraction}, we have $\|\mathcal{T}Q_t-\mathcal{T}V\|_\pi\leq \gamma\|Q_t-V\|_\pi$, which implies the result.
\end{proof}

\textbf{Part (2)} For any $t$, we have:
$$\bE[\Delta_t \big(V(x_t)-\nabla_W^\top Q_0(x_t)\overline{W}\big);\cE_t] \leq \frac{4\overline{\nu}}{\sqrt{m}} \sqrt{\bE[\|Q_t-V\|_\pi^2; \cE_t]} .$$

\begin{proof}
Let $\nabla_W^\top Q_0\overline{W} = \big[\nabla_W^\top Q_0(x)\overline{W}\big]_{x\in\mathcal{X}}$. Then, for any $t$, we have:
\begin{multline*}\bE_t[\Delta_t \big(V(x_t)-\nabla_W^\top Q_0(x_t)\overline{W}\big)] \\= \bE_t[(\cT Q_t(x_t)-Q_t(x_t))\big(V(x_t)-\nabla_W^\top Q_0(x_t)\overline{W}\big)]
\end{multline*}
By using Cauchy-Schwarz inequality, we have:
\begin{multline*}\bE_t[(\cT Q_t(x_t)-Q_t(x_t))\big(V(x_t)-\nabla_W^\top Q_0(x_t)\overline{W}\big)] \\ \leq \|\cT Q_t-Q_t\|_\pi\sqrt{\bE_t[(V(x_t)-\nabla_W^\top Q_0(x_t)\overline{W})^2]}.
\end{multline*}
Then, by using Claim \ref{cl:bellman},
\begin{equation*}
    \|\cT Q_t-Q_t\|_\pi \leq (1+\gamma) \|Q_t-V\|_\pi.
\end{equation*}
% &\leq \sqrt{\bE_t[(\cT Q_t(x_t)-Q_t(x_t))^2]}\sqrt{\bE_t[(V(x_t)-\nabla_W^\top Q_0(x_t)\overline{W})^2]},\\
% &\leq (1+\gamma) \|Q_t-V\|_\pi \|V-\nabla_W^\top Q_0\overline{W}\|_\pi,
% \end{align*} 
By the law of iterated expectations,
\begin{multline*}
    \bE[\Delta_t \big(V(x_t)-\nabla_W^\top Q_0(x_t)\overline{W}\big);\cE_t] \\ \leq (1+\gamma)\bE[\|Q_t-V\|_\pi\bI_{\cE_t} \|V-\nabla_W^\top Q_0\overline{W}\|_\pi],\\
\end{multline*}
since $\bI_{\cE_t} \in \mathcal{F}_{t-1}$. Hence, by Cauchy-Schwarz inequality, we have the following:
\begin{multline}
    \bE[\Delta_t \big(V(x_t)-\nabla_W^\top Q_0(x_t)\overline{W}\big);\cE_t] \\ \leq 2 \sqrt{\bE[\|Q_t-V\|_\pi^2;\cE_t]}\sqrt{\bE[\|V-\nabla_W^\top Q_0\overline{W}\|_\pi^2]}.
    \label{eqn:initial-error}
\end{multline}
In the following, we will bound $\sqrt{\bE[\|V-\nabla_W^\top Q_0\overline{W}\|_\pi^2]}$. For any $x\in\cX$, we have:
\begin{equation}
    V(x) - \nabla_W^\top Q_0(x)\overline{W} =\frac{1}{m}\sum_{i=1}^m\Big( V(x)-\widehat{V}_i(x)\Big).\label{eqn:dist-ntk-ntrf}
\end{equation}
where $\widehat{V}_i(x) = \mathbb{I}\{W_i^\top(0)x \geq 0\}v^\top(W_i(0))x$. Recall that $V(x) = \bE[\mathbb{I}\{W_i^\top(0)x \geq 0\}v^\top(W_i(0))x] = \bE[\widehat{V}_i(x)]$ by Assumption \ref{assumption:realizability}. Hence, for any $i\in[m]$, $$\bE[V(x)-\wV_i(x)] = 0,$$ and for $i, j \in [m/2]$, we have:
\begin{equation*}
    Cov\Big(\wV_i(x), \wV_j(x)\Big) \\\leq \bI\{i=j\}\bE[\|v(W_1(0))\|_2^2]. 
\end{equation*}
Under symmetric initialization, $W_i(0) = W_{i+m/2}(0)$ for all $i\in[m/2]$. Therefore, by using the above result along with Fubini's theorem, we have:
\begin{align}
    \nonumber \bE\|V-&\nabla_W^\top Q_0\overline{W}\|_\pi^2 \\\nonumber &= \bE\Big[\int_{x\in\cX}\Big(\frac{1}{m}\sum_{i=1}^m\Big( V(x)-\wV_i(x)\Big)\Big)^2\pi(dx)\Big],\\
    \nonumber &\leq\int_{x\in\cX}\frac{4}{m^2}\sum_{i=1}^m\bE\Big[\Big|V(x)-\wV_i(x)\Big|^2\Big]\pi(dx),\\
    &\leq 4\int_{x\in\cX} \frac{\bE[\|v(W_1(0))\|_2^2]}{m}\pi(dx) \leq \frac{4\overline{\nu}^2}{m}, \label{eqn:rkhs-inequality}
\end{align}
since $Var(\wV_i(x)) \leq \bE[\|v(W_1(0))\|_2^2] \leq \overline{\nu}^2$ by Assumption \ref{assumption:realizability} and $\|x\|_2\leq 1$ for all $x\in\cX$ by Assumption \ref{assumption:norm}. The extra factor is due to the symmetric initialization. By substituting \eqref{eqn:rkhs-inequality} into \eqref{eqn:initial-error}, we have:
\begin{multline*}
    \bE\Big[\Delta_t \big(V(x_t)-\nabla_W^\top Q_0(x_t)\overline{W}\big); \cE_t\Big] 
        \\\leq \frac{4\overline{\nu}}{\sqrt{m}}\sqrt{\bE[\|Q_t-V\|_\pi^2;\cE_t]}.
\end{multline*}
\end{proof}
\textbf{Part (3)} Let
\begin{equation*}
    \overline{U}_i = a_i\frac{v(W_i(0))}{\sqrt{m}},i\in[m],
\end{equation*}
with $\overline{U} = [\overline{U}_i]_{i\in[m]}$, which implies $\overline{W} = W(0)+\overline{U}$. Note that under symmetric initialization, $\nabla_W Q_0^\top(x)W(0) = Q_0(x)= 0$ for all $x\in\cX$. Then, for any $t$, we have:
\begin{multline}
    \bE[\Delta_t \big(\nabla_WQ_0(x_t)-\nabla_WQ_t(x_t)\big)^\top\overline{U};\cE_t] \\\leq \frac{{4\rkhsnorm}\big(\lambda+\ell(m,\delta)\big)}{\sqrt{m}}z_t,
    \label{eqn:lazy-tr-a}
    \end{multline}
    and
    \begin{multline}
      \bE[\Delta_t \big(\nabla_WQ_0(x_t)-\nabla_WQ_t(x_t)\big)^\top W(0);\cE_t] \\ \leq \frac{{4\lambda}\big(\lambda+\ell(m,\delta)\big)}{\sqrt{m}}z_t,
      \label{eqn:lazy-tr-b}
    \end{multline}
    with probability at least $1-\delta$ over the random initialization.
    
\begin{proof}
In order to prove \eqref{eqn:lazy-tr-a}, we have the following bound by using Claim \ref{cl:bellman}:
\begin{multline}
    \bE_t[\Delta_t\cdot \big(\nabla_WQ_0(x_t)-\nabla_WQ_t(x_t)\big)^\top\overline{U}] \bI_{\cE_t}\\ \leq (1+\gamma)\|Q_t-V\|_\pi \|\nabla_W^\top Q_t\overline{U}-\nabla_W^\top Q_0\overline{U}\|_\pi \label{eqn:lazy-tr-a-exp}
    % \end{split} \sqrt{\int_{x\in\cX}\Big(\big(\nabla_WQ_0(x_t)-\nabla_WQ_t(x_t)\big)^\top\overline{U}\Big)^2\pi(dx)}\bI_{\cE_t}.
\end{multline}
For any $x\in\cX$, we have:
\begin{multline*}
    \big(\nabla_WQ_0(x)-\nabla_WQ_t(x)\big)^\top\overline{U} \\= \sum_{i\in[m]}\Big(\mathbb{I}\{W_i^\top(0)x \geq 0\}-\mathbb{I}\{W_i^\top(t)x \geq 0\}\Big)\frac{v^\top(W_i(0))x}{m}.
\end{multline*}
Let 
\begin{equation}
    S_x(t) = \Big\{i\in[m]:\mathbb{I}\{W_i^\top(0)x \geq 0\}\neq \mathbb{I}\{W_i^\top(t)x \geq 0\}\Big\}.
    \label{eqn:separation-set}
\end{equation} 
For any $x\in\cX$ and $i\in S_x(t)$, we have: $$|W_i^\top(0)x| \leq |W_i^\top(0)x-W_i^\top(t)x| \leq \|W_i(0)-W_i(t)\|_2,$$ since $i\in S_x(t)$ implies $W_i^\top(0)x$ and $W_i^\top(t)x$ have different signs. Therefore, we have the following relation:
\begin{align}
\begin{aligned}
    S_x(t) &\subset \Big\{i\in[m]:|W_i^\top(0)x| \leq \|W_i(0)-W_i(t)\|_2\Big\},\\
    &\subset \Big\{i\in[m]:|W_i^\top(0)x| \leq \lambda / \sqrt{m}\Big\},
    \end{aligned}
    \label{eqn:separation-set-inc}
\end{align}
for any $t < t_1$. With this definition, we have:
\begin{multline}
    \big|\big(\nabla_WQ_0(x)-\nabla_WQ_t(x)\big)^\top\overline{U}| \\ \leq \frac{1}{m}\sum_{i\in[m]}\mathbb{I}\{i\in S_x(t)\}\rkhsnorm \leq \frac{4\rkhsnorm}{m}\widetilde{S}(x).
    \label{eqn:lazy-tr-a-exp-ii}
\end{multline}
since $v(w) \leq \rkhsnorm$ for any $w\in\bR^d$ by Assumption \ref{assumption:realizability}, where 
\begin{equation} \widetilde{S}(x) = \sum_{i=1}^{m/2}\bI\Big\{|W_i^\top(0)x|\leq\lambda/\sqrt{m}\Big\},
\label{eqn:tilde-x}
\end{equation}
for any $x\in\cX$. By Claim \ref{claim:glivenko-cantelli}, under $E_1\cap\{t<t_1\}$, we have:
\begin{equation}
    \frac{2\widetilde{S}(x)}{m} \leq \frac{\lambda}{\sqrt{m}} + \frac{\sqrt{2}\ell(m/2,\delta)}{\sqrt{m}}.
    \label{eqn:glivenko}
\end{equation}

Therefore, we can bound \eqref{eqn:lazy-tr-a-exp} as follows:
\begin{multline*}
    \bE_t[\Delta_t \big(\nabla_WQ_0(x_t)-\nabla_WQ_t(x_t)\big)^\top\overline{U}] \bI_{\cE_t} \\ \leq \frac{{4\rkhsnorm}(\lambda+\ell(m,\delta))}{\sqrt{m}}\|Q_t-V\|_\pi \bI_{\cE_t}.
\end{multline*}
By taking expectation and using Cauchy-Schwarz inequality, we obtain:
\begin{equation*}
    \bE[\Delta_t \big(\nabla_WQ_0(x_t)-\nabla_WQ_t(x_t)\big)^\top\overline{U}] \leq \frac{{4\rkhsnorm}\big(\lambda+\ell(m,\delta)\big)}{\sqrt{m}}z_t.
\end{equation*}

In order to prove \eqref{eqn:lazy-tr-b}, we use Claim \ref{cl:bellman} to obtain the following inequality:
\begin{multline}
    \bE_t[\Delta_t \big(\nabla_WQ_0(x)-\nabla_WQ_t(x)\big)^\top W(0)] \\ \leq 2\|Q_t-V\|_\pi \|\big(\nabla_W^\top Q_0W(0) -\nabla_W^\top Q_tW(0)\|_\pi.
    \label{eqn:lazy-tr-b-exp}
\end{multline}
For any $x\in\cX$, we have:
\begin{multline*}
    \big(\nabla_WQ_0(x)-\nabla_WQ_t(x)\big)^\top W(0) \\= \frac{1}{\sqrt{m}}\sum_{i\in[m]}a_i\Big(\mathbb{I}\{W_i^\top(0)x \geq 0\}-\mathbb{I}\{W_i^\top(t)x \geq 0\}\Big)W_i^\top(0)x.
\end{multline*}
Recall the definition of $S_x(t)$ in \eqref{eqn:separation-set}. By using triangle inequality:
\begin{multline*}
    \big|\big(\nabla_WQ_0(x)-\nabla_WQ_t(x)\big)^\top W(0)| \\\leq \frac{1}{\sqrt{m}}\sum_{i\in[m]}\mathbb{I}\{i\in S_x(t)\}\cdot |W_i^\top(0)x|.
\end{multline*}
For any $x\in\cX$ and $i\in S_x(t)$, we have: $$|W_i^\top(0)x| \leq |W_i^\top(0)x-W_i^\top(t)x| \leq \|W_i(0)-W_i(t)\|_2,$$ since $i\in S_x(t)$ implies $W_i^\top(0)x$ and $W_i^\top(t)x$ have different signs. The correlation between $\bI\{i\in S_x(t)\}$ and $\|W_i(t)-W_i(0)\|_2$ creates the main problem in the proof, which we resolve under the event $\{t < t_1\}$. For $t < t_1$, we have $\|W_i(0)-W_i(t)\|_2 \leq \lambda/\sqrt{m}$. Thus, we have:
\begin{align*}
    \big|\big(\nabla_WQ_0(x)-\nabla_WQ_t(x)\big)^\top W(0)| &\leq \frac{\lambda}{m}\sum_{i\in[m]}\mathbb{I}\{i\in S_x(t)\},\\
    &\leq  \frac{\lambda}{m}|S_x(t)| \leq \frac{4\lambda}{m}\widetilde{S}(x),
\end{align*}
where $\widetilde{S}(x)$ is defined in \eqref{eqn:tilde-x}. Using Claim \ref{claim:glivenko-cantelli} similar to \eqref{eqn:glivenko}, under $E_1\cap\{t<t_1\}$, we have:
\begin{equation*}
    \bE[\Delta_t \big(\nabla_WQ_0(x_t)-\nabla_WQ_t(x_t)\big)^\top\overline{U}] \leq \frac{{4\lambda}\big(\lambda+\ell(m,\delta)\big)}{\sqrt{m}}z_t.
\end{equation*}
% Substituting this into \eqref{eqn:lazy-tr-b-exp}, we obtain the following inequality:
% \begin{equation*}
%     \bE_t[\Delta_t \big(\nabla_WQ_0(x)-\nabla_WQ_t(x)\big)^\top W(0)]\bI\{t<t_1\} \\ \leq 2\lambda\|Q_t-V\|_\pi\bI\{t<t_1\} \sqrt{\frac{\bE_t[\widetilde{S}(x_t)]}{m}},
% \end{equation*}
% where $\bE_t[\widetilde{S}(x_t)] = \int_{x\in\cX}\widetilde{S}(x)\pi(dx)$. By using Cauchy-Schwarz inequality, the inequality $\bE[\widetilde{S}(x)]\leq\lambda\sqrt{m}$ for any $x\in\cX$ by \eqref{eqn:anti-concentration}, and Fubini's theorem as in \eqref{eqn:fubini}, we conclude the proof. As a final remark, we note that bounds on the expected deviation $\bE\|W_i(t)-W_i(0)\|_2$ do not suffice to prove this bound (or any bound that vanishes for large $m$) because of the correlation between $\mathbb{I}\{W_i^\top(0)x \geq 0\}-\mathbb{I}\{W_i^\top(t)x \geq 0\}$ and $W_i^\top(0)x$ for $x\in\cX$ and $i\in[m]$, therefore we analyze under the event $\{t < t_1\}$, which holds with high probability.
% 
\end{proof}

\subsection{Proximity of $\overline{Q}_T$ and $\frac{1}{T}\sum_{t<T}Q_t$}
In this section, we will show that the output of Algorithm \ref{alg:neural-td}, $\overline{Q}_T(x) = Q(x;\frac{1}{T}\sum_{t<T}W(t),a)$, is close to $\frac{1}{T}\sum_{t<T}Q_t(x)$ in expectation, which will prove that $\overline{Q}_T$ achieves the target error. The idea is based on \cite{cai2019neural}, and aims to use the linear approximation $\nabla_W^\top Q_0(x)\widehat{W}(T-1)$ as an auxiliary function to show the proximity of $\overline{Q}_T$ and $\frac{1}{T}\sum_{t<T}Q_t$.
\begin{proposition}
Let $\widetilde{W}\in \bR^{md}$ be a (random) vector of parameters. Also, let $\widehat{Q}(x) = Q(x; \widetilde{W},a)$ and $\widehat{Q}_0(x) = \nabla_W^\top Q_0(x)\widetilde{W}$ for any $x\in\cX$, and the event $\mathcal{A} = \{\max_{i\in[m]}\|\widetilde{W}_i-W_i(0)\|_2\leq \frac{\lambda}{\sqrt{m}}\}\cap E_1$. Then, we have the following inequality:
\begin{align*}
    \bE[\|\widehat{Q}-\widehat{Q}_0\|_\pi;\mathcal{A}] \leq \frac{\lambda(\lambda + \ell(m,\delta))}{\sqrt{m}}\leq \frac{\epsilon}{2}.
\end{align*}
Consequently, we have:
\begin{equation}
    \bE\Big[\Big\|\overline{Q}_T-\frac{1}{T}\sum_{t<T}Q_t\Big\|_\pi;\cE_T\Big] \leq \epsilon.
\end{equation}
\label{prop:proximity}
\end{proposition}
\begin{proof}
First, note that the difference of $\widehat{Q}$ and $\widehat{Q}_0$ can be written as follows:
\begin{multline*}
    |\widehat{Q}(x)-\widehat{Q}_0(x)| \\\leq \frac{1}{\sqrt{m}}\sum_{i\in[m]}\big|\bI\{\widetilde{W}_i^\top x \geq 0\}-\bI\{W_i^\top(0) x \geq 0\}\big|\cdot |\widetilde{W}^\top_i x|,
\end{multline*}
for any $x\in\cX$. Let $$S_x = \Big\{i\in[m]:\mathbb{I}\{W_i^\top(0)x \geq 0\}\neq \mathbb{I}\{\widetilde{W}_i^\top x \geq 0\}\Big\}.$$ Then, we have:
\begin{align*}
    |\widehat{Q}(x)-\widehat{Q}_0(x)| &\leq \frac{1}{\sqrt{m}}\sum_{i\in[m]}\bI\{i\in S_x\}|\widetilde{W}_i^\top x| \\ &\leq \frac{1}{\sqrt{m}}\sum_{i\in[m]}\bI\{i\in S_x\}\|\widetilde{W}_i-W_i(0)\|_2.
\end{align*}
since $i\in S_x$ implies $|\widetilde{W}_i^\top x| \leq |\widetilde{W}_i^\top x-W_i^\top(0) x| \leq \|\widetilde{W}_i-W_i(0)\|_2$. Similarly, we have: $|W_i^\top(0)x|\leq \|\widetilde{W}_i-W_i(0)\|_2$. Then, we have:
\begin{align*}
    |\widehat{Q}(x)-\widehat{Q}_0(x)|\bI_{\mathcal{A}} &\leq \frac{\lambda}{m}|S_x|\bI_{\mathcal{A}}\\ &\leq \frac{\lambda(\lambda + \ell(m,\delta))}{\sqrt{m}}.
\end{align*}
% \textcolor{red}{SS: the next statement does not seem like being used anywhere in the proofs, infact the |S_x| is bounded using equation (45) if i can tell correctly}
% Therefore,
% \begin{align*}
%     |\widehat{Q}(x)-\widehat{Q}_0(x)|^2\bI_{\mathcal{A}} &\leq \frac{\lambda^2}{m}|S_x|\bI_{\mathcal{A}},\\
%     &\leq \frac{\lambda^2}{m}\sum_{i\in[m]}\bI\{|W_i^\top(0) x| \leq \lambda/\sqrt{m}\}.
% \end{align*}
Taking the expectation and using Jensen's inequality, we have: 
\begin{align*}\bE[\|\widehat{Q}-\widehat{Q}_0\|_\pi;\mathcal{A}] &\leq \sqrt{\bE[\|\widehat{Q}_T-\widehat{Q}_0\|_\pi^2;\mathcal{A}]}\\&\leq \frac{\lambda(\lambda + \ell(m,\delta))}{\sqrt{m}},
\end{align*}
which concludes the proof of the first claim.

In order to prove the second claim, consider $\widetilde{W} = \widehat{W}(T-1)=\frac{1}{T}\sum_{t<T}W(t)$, and note that $\widehat{W}(T-1)\in\mathcal{F}_{T-1}$ and $\cE_T\subset \mathcal{A}$ by definition. Therefore, the first part implies the following:
\begin{equation}
    \bE[\|\overline{Q}_T-\nabla_W^\top Q_0\widehat{W}(T-1)\|_\pi;\cE_T] \leq \frac{\lambda(\lambda + \ell(m,\delta))}{\sqrt{m}}.
    \label{eqn:proximity-a}
\end{equation}
with the usual notation $\nabla_W^\top Q_0\widehat{W}(T-1)=\big[\nabla_W^\top Q_0(x)\widehat{W}(T-1)\big]_{x\in\cX}$.
Finally, we have: \begin{multline*}\bE[\|\frac{1}{T}\sum_{t<T}Q_t-\nabla_W^\top Q_0\widehat{W}(T-1)\|_\pi;\cE_T] \\ \leq \frac{1}{T}\sum_{t<T}\bE[\|Q_t-\nabla_W^\top Q_0 W(t)\|_\pi;\cE_T],\end{multline*} 
by Jensen's inequality. For any $t<T$, letting $\widetilde{W}=W(t)$, and noting that $\cE_T \subset \mathcal{A}$, we have $\bE[\|Q_t-\nabla_W^\top Q_0 W(t)\|_\pi;\cE_t] \leq \frac{\lambda(\lambda + \ell(m,\delta))}{\sqrt{m}}$ by using the first part of the proposition, which implies:
\begin{equation}
    \bE[\|\frac{1}{T}\sum_{t<T}Q_t-\nabla_W^\top Q_0\widehat{W}(T-1)\|_\pi;\cE_T] \leq \frac{\lambda(\lambda + \ell(m,\delta))}{\sqrt{m}}.
    \label{eqn:proximity-b}
\end{equation}
Using \eqref{eqn:proximity-a}, \eqref{eqn:proximity-b} and triangle inequality together, we conclude that $$\bE[\|\frac{1}{T}\sum_{t<T}Q_t-\overline{Q}_T\|_\pi;\cE_T]\leq \epsilon,$$ with the choice of parameters in Theorem \ref{thm:neural-td-exp}.
\end{proof}

\section{Proof of Theorem \ref{thm:mn-td}}\label{app:mn-ntd}
The proof of Theorem \ref{thm:mn-td} consists of the same steps as Theorem \ref{thm:neural-td-exp}, but it is simpler because the growth of $\|W(t)-\overline{W}\|_2$ is controlled by the max-norm constraint. In the first step, we will prove a Lyapunov drift bound.

\subsection{Lyapunov Drift Bound}
First, note that for any $R>0$ and $m\in\mathbb{N}$, $$\mathcal{G}_{m, R} = \Big\{w\in\mathbb{R}^{md}:\|W_i(0)-w_i\|_2 \leq \frac{R}{\sqrt{m}},\forall i\in[m]\Big\},$$ is the Cartesian product of convex sets $\mathcal{G}_{m,R}^i$, which is convex. This leads to the following result.

\begin{lemma}
For any $t\geq 0$ and $R \geq \rkhsnorm$, we have the following inequalities:
\begin{align}
\begin{aligned}
    \bE[\|W(t+1)&-\overline{W}\|_2^2;E_1] \leq \bE[\|W(t)-\overline{W}\|_2^2;E_1]\\&- 2\alpha(1-\gamma)z_t^2+ \alpha^2(1+2R)^2
    \\&+8\alpha z_t\Big(\frac{\rkhsnorm+\big(\rkhsnorm+R\big)\big(R+\ell(m,\delta)\big)}{\sqrt{m}}\Big),
    \end{aligned}
\end{align}
where $\overline{W}$ is as defined in \eqref{eqn:opt-parameter}, $z_t = \sqrt{\bE[\|Q_t-V\|_\pi^2;E_1]}$.
\label{lemma:drift-bound-mn-ntd}
\end{lemma}

\begin{proof}
First, note that $W(t+1) = \Pi_{\mathcal{G}_{m,R}}W(t+1/2)$ by the update rule in \eqref{eqn:max-norm}, and $\mathcal{G}_{m,R}$ is a convex set. Also, note that $R \geq \rkhsnorm$ implies $\overline{W}\in\mathcal{G}_{m,R}$. Therefore, we have:
\begin{align*}
    \|W(t+1)-\overline{W}\|_2^2 &= \|\Pi_{\mathcal{G}_{m,R}}W(t+1/2)-\Pi_{\mathcal{G}_{m,R}}\overline{W}\|_2^2,\\
    &\leq \|W(t+1/2)-\overline{W}\|_2^2,
\end{align*}
which follows since projection is a non-expansive operation for convex subsets. Since $W(t+1/2) = W(t)+\alpha g_t$ and $\|g_t\|_2 \leq 1+2R$ by \eqref{eqn:gradient-norm}, we have:
\begin{multline*}
    \bE_t\|W(t+1)-\overline{W}\|_2^2 \leq \|W(t)-\overline{W}\|_2^2 + 2\alpha\bE_t[g_t^\top](W(t)-\overline{W})\\+\alpha^2(1+2R)^2.
\end{multline*}
Then, the proof follows by multiplying both sides by $\bI_{E_1}$, taking expectation, and using Proposition \ref{prop:drift-td} with $\lambda$ replaced by $R$ since $\|W_i(t)-W_i(0)\|_2 \leq R/\sqrt{m}$ for all $i\in[m]$ and $t_1 = \infty$.
% \textcolor{red}{SS: Since proposition 1 uses the conditioning $t<T$, maybe its better to write `using a similar proof technique as proposition $1$'?}
\end{proof}

\subsection{Error Bound}
Note that by the choices of step-size $\alpha$ and network width $m$, we have:
\begin{equation*}
    \alpha^2(1+2R)^2 = \alpha(1-\gamma)\epsilon^2,\\
\end{equation*}
and
\begin{equation*}
    \frac{\rkhsnorm+\big(\rkhsnorm+R\big)\big(R+\ell(m,\delta)\big)}{\sqrt{m}} \leq \epsilon(1-\gamma)/4.
\end{equation*}
Using these in Lemma \ref{lemma:drift-bound-mn-ntd}, we have:
\begin{multline*}
    \bE[\|W(t+1)-\overline{W}\|_2^2;E_1] \leq \bE[\|W(t)-\overline{W}\|_2^2;E_1] \\- \alpha(1-\gamma)\Big(z_t-\epsilon\Big)^2 + 2\alpha(1-\gamma)\epsilon^2.
    % \label{eqn:drift-mn-ntd}
\end{multline*}
By telescoping sum over $t=0,1,\ldots,T-1$, the above inequality yields:
\begin{align*}
    \frac{1}{T}\sum_{t<T}(z_t-\epsilon)^2 &\leq \frac{\bE[\|W(0)-\overline{W}\|_2^2;E_1]}{\alpha(1-\gamma)T}+2\epsilon^2,\\
    &\leq \frac{\rkhsnorm^2}{\alpha(1-\gamma)T}+2\epsilon^2.
\end{align*}
By using Jensen's inequality,
\begin{equation*}
    \Big(\frac{1}{T}\sum_{t<T}z_t-\epsilon\Big)^2 \leq \frac{\rkhsnorm^2}{\alpha(1-\gamma)T}+2\epsilon^2.
\end{equation*}
The above inequality yields:
\begin{align*}
    \frac{1}{T}\sum_{t<T}\bE[\|Q_t-V\|_\pi;E_1]&\leq \frac{1}{T}\sum_{t<T}z_t \\ & \leq \frac{\rkhsnorm}{\sqrt{\alpha(1-\gamma)T}}+3\epsilon.
\end{align*}
We conclude the proof by using Proposition \ref{prop:proximity}.

\end{document}